\documentclass[twoside]{article}

\usepackage[accepted]{aistats2024}
\usepackage[round]{natbib}

\usepackage{amsthm}
\usepackage{mathtools}
\usepackage{amsmath}
\usepackage{amssymb}
\usepackage{float}
\usepackage{notation}
\usepackage{booktabs}
\usepackage{graphicx}
\usepackage{stfloats}  
\usepackage{algpseudocode}
\usepackage{subfig}
\usepackage{xr}
\usepackage{algorithm}
\usepackage{MnSymbol}
\usepackage{tabularx}
\usepackage{comment}

\makeatletter
\newcommand*{\addFileDependency}[1]{
\typeout{(#1)}
\@addtofilelist{#1}
\IfFileExists{#1}{}{\typeout{No file #1.}}
}
\makeatother
\newcommand*{\myexternaldocument}[1]{%
\externaldocument{#1}%
\addFileDependency{#1.tex}%
\addFileDependency{#1.aux}%
}
\myexternaldocument{supplement}

\begin{document}
\twocolumn[

\aistatstitle{Optimal Budgeted Rejection Sampling for Generative Models}

\aistatsauthor{Alexandre Verine \And Muni Sreenivas Pydi   \And  Benjamin Negrevergne \And Yann Chevaleyre }

\aistatsaddress{ LAMSADE, CNRS,\\
Université Paris-Dauphine,\\
Université PSL,\\
Paris, France. \And LAMSADE, CNRS,\\
Université Paris-Dauphine,\\
Université PSL,\\
Paris, France. \And LAMSADE, CNRS,\\
Université Paris-Dauphine,\\
Université PSL,\\
Paris, France. \And LAMSADE, CNRS,\\
Université Paris-Dauphine,\\
Université PSL,\\
Paris, France. } ]

\begin{abstract}

Rejection sampling methods have recently been proposed to improve the performance of discriminator-based generative models. 
However, these methods are only optimal under an unlimited sampling budget, and are usually applied to a generator trained independently of the rejection procedure. 
We first propose an Optimal Budgeted Rejection Sampling (OBRS) scheme that is provably optimal with respect to \textit{any} $f$-divergence between the true distribution and the post-rejection distribution, for a given sampling budget. Second, 
we propose an end-to-end method that incorporates the sampling scheme into the training procedure to further enhance the model's overall performance. 
Through experiments and supporting theory, we show that the proposed methods are effective in significantly improving the quality and diversity of the samples.  

\end{abstract}

\section{INTRODUCTION}
\begin{figure*}[t]
    \centering
    \includegraphics[width=\textwidth]{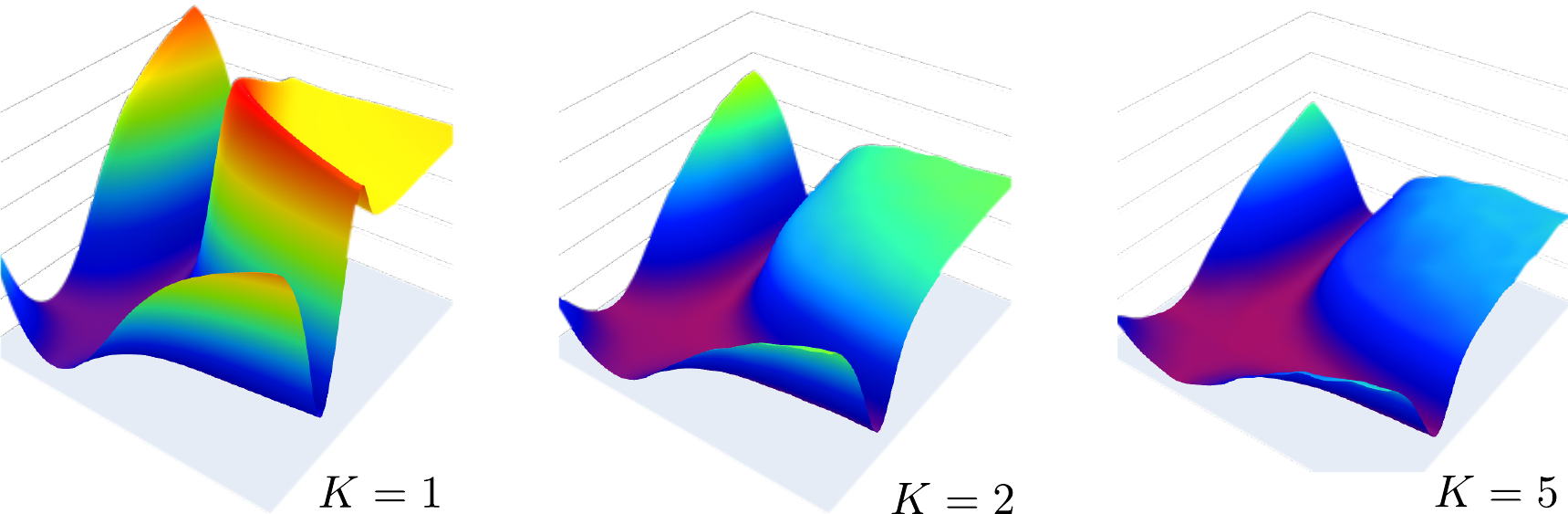}
    \caption{The loss landscape in the parameter domain of a GAN trained on MNIST. The x-axis and y-axis are random directions in the parameter space. The loss is between the target distribution $P$ and the post-rejection distribution. There are three cases: no rejection ($K=1$), $50\%$  acceptance rate ($K=2$) and $20\%$  acceptance rate ($K=5$). OBRS not only reduces loss, but also flattens out the loss landscape and helps avoid  local minima.}\label{fig:MNISTSmooth}
\end{figure*}
Generative Adversarial Networks (GANs)  have significantly improved generation of complex, high dimensional data. In the original paper by \citet{goodfellow_generative_2014}, GANs are trained to  minimize the Jensen-Shannon divergence between true distribution $P$ and a distribution $\whP$ induced by a generator $G$.
Since $P$ is generally unknown, the divergence between $P$ and $\whP$ is estimated using a discriminator $T$, i.e. a function that discriminates available samples from $P$ and samples generated from $\whP$. In practice $T$ and $G$ are  represented using neural networks and trained simultaneously to estimate the divergence and to minimize it. 
In this paper, we consider the more general framework of $f$-GAN introduced by \citet{nowozin_f-gan_2016}, which  can be used to minimize {\em any} \fdiv between $P$ and  $\whP$, including the Jensen-Shannon divergence, the Kullback-Leibler divergence or other divergences (See Table~\ref{tab:fdiv}).

In most settings, the discriminator is not involved in the generation of new samples beyond the training phase (i.e. it is discarded after training). Building on this observation, several methods  such as  Discriminator Rejection Sampling (DRS) \citep{azadi_discriminator_2019} or Metropolis-Hastings GAN \citep{turner_metropolis-hastings_2019} have demonstrated how to combine $G$ \emph{and} $T$ using rejection sampling, in order to generate better samples than the ones generated using $G$ alone. 
In the rest of this paper, we call $\wtP$ the distribution resulting from $G$ enhanced with rejection sampling.

Unfortunately, these methods suffer from several limitations. First they are only provably optimal when the sampling budget is unlimited. In practice, users have to limit the rejection rate to obtain samples in reasonable time through various empirical means (e.g. by capping the number of iterations of the sampling algorithm). This strategy may not  yield the best possible sample for the given budget, an observation that leads to the first question that motivated our contribtion. 

\begin{question}\label{question:q1}
    How to devise a method that generates the best quality sample under a fixed rejection budget?  
\end{question}

Another important limitation is that, since examples are sampled from $\wtP$ rather than $\whP$, the objective should be to minimize the divergence between $P$ and $\wtP$ rather than the divergence between $P$ and $\whP$. This raises a  second research question that we address in this paper:


\begin{question}\label{question:q2}
    Can we train a generator $G$ that directly minimizes $\Df(P\Vert \wtP)$ instead of $\Df(P\Vert \whP)$ ?
\end{question}

In this paper, we address Question~\ref{question:q1}\&\ref{question:q2}, by making following contributions:
\begin{itemize}
    \item    
    We introduce ORBS, a method that can be used to find an  {\em acceptance function} required to reject/accept samples from $\whP$ and show in Theorem~\ref{thm:optrej} that this function induces the optimal distribution $\wtP$ under a budget, for \emph{any} \fdiv. 
    \item We characterize the improvement of $\wtP$ over $\whP$ in terms of Precision and Recall~\citep{sajjadi_assessing_2018} in Theorem~\ref{thm:improvalpha}.
    \item We propose a method to train a generator $G$ to directly minimize an \fdiv between $P$ and $\wtP$, and we discuss the potential benefits of our method. For example, in Figure~\ref{fig:MNISTSmooth}, we illustrate how OBRS can flatten the loss landscape.    
\end{itemize}

\textbf{Notation: } For the rest of the paper, we use $\cX\subseteq \mathbb R^d$ to refer to the data space. We use $\cP(\cX)$ to denote the set of probability measures on $\cX$ defined on a measure space with the Borel $\sigma$-algebra. We use capital letters to denote probability measures (for e.g., $P\in \cP(\cX)$) and small letters to denote their densities (for e.g., $p(\vx)$ for $\vx\in \cX$).

\section{BACKGROUND}
\subsection{\fdivs}

The framework of \fdivs can be used to specify a variety of divergences between two probability distributions. 
An \fdiv is fully characterized by a convex and lower semi-continuous function $f:\reals^+ \to \reals$ that satisfies  $f(1)=0$. Given $f$ and two probability distributions  $P$ and $\whP\in \cP(\cX)$, the \fdiv  between $P$, $\whP$ (denoted $\Df (P\Vert\whP )$) is defined as follows:
\begin{align}
\label{eq:fdiv}\Df (P\Vert\whP ) = \E_{\vx\sim \wh P}\left[  f\left(\frac{p(\vx)}{\whp(\vx)}\right)\right].
\end{align}

(We assume that $P$ is absolutely continuous with w.r.t. $\whP$.) Several notable statistical divergences, such as the Kullback-Leibler (KL) divergence ($\KL$), the reverse KL divergence ($\rKL$), or the Total Variation ($\TV$), belong to the class of \fdivs. 
A  overview is provided in Table~\ref{tab:fdiv}.

A key property of \fdivs is that every \fdiv $\Df$ admits a dual variational form \citep{nguyen_surrogate_2009}:
\begin{equation}
\begin{aligned}
\label{eq:dual}
\Df(P \Vert \wh P)
=\sup_{\T \in \cal T} \E_{\vx\sim P}\left[T(\vx) \right] - \E_{\vx\sim \wh P}\left[ f^*(T(\vx))\right],
\end{aligned}
\end{equation}
where $\cT$ be the set of all measurable functions $T: \cX\to \mathbb{R}$ and $f\s(t) \coloneqq \sup_{u \in \reals} \left\{tu-f(u)\right\}$ is the convex conjugate of $f$. Specifically, the function $\Topt\in \cal T$ that yields the supremum in \eqref{eq:dual} can be used to determine the likelihood ratio $\ropt$ as follows.
\begin{equation}\label{eq:densityestimation}
\ropt(\vx) =  \nabla f\s \left(\Topt(\vx)\right) = \frac{p(\vx)}{\wh p (\vx)}.
\end{equation}

\begin{table*}[t]
\caption{List of common \fdivs. The generator $f$ is given with its Fenchel conjugate $f^*$. The optimal discriminator $\Topt$ is given to compute the likelihood ratio $p(\vx)/\whp(\vx) = \nabla f^*(\Topt(\vx))$. }
\label{tab:fdiv}
\begin{sc}
\begin{center}

\begin{tabular*}{\textwidth}{l @{\extracolsep{\fill}} ccccc}
\toprule
Divergence & Notation & $f(u)$ & $f\s(t)$ & $\Topt(\vx)$ \\
\midrule \addlinespace[0.5em]
KL & $\KL $ & $u\log u$& $\exp(t-1)$ & $1 + \log p({\vx})/\whp(\vx)$  \\ \addlinespace[0.4em]
GAN & $\GAN $ & $u\log u - (u+1)\log(u+1)$& $-\log(1-\exp (t) ) $ & $p(\vx)/\left(p(\vx)+\whp(\vx)\right)$  \\ \addlinespace[0.4em]
PR & $\DPR$ & $\max(\lambda u , 1) - \max(\lambda, 1)$ & $t/\lambda$ &  $\lambda \sign\left\{ p(\vx)/\whp(\vx)-1\right\}$   \\ \addlinespace[0.4em]
\bottomrule
\end{tabular*}
\end{center}
\end{sc}
\end{table*}

\subsection{$f$-GAN, a generalization of GAN}\label{subsec:fgan}
Let $\cal G$ be the set of all measurable functions $G: \cal Z \to\cal X$, where $\cal Z$ is the latent space and $\cX$ is the data space. 
In the $f$-GAN framework, 
the generator $G\in \cal G$ is used to transform samples from the latent distribution $Q\in \cP(\cal Z)$ (typically a multivariate Gaussian) into data samples following the data distribution $\whP_G\in \cP(\cal X)$ . $G$ is chosen to minimize the $f$-divergence $\Df(P \Vert \wh P_G)$ 

Since $P$ is usually not available, a discriminator $T: \setX \to \mathbb R$ is used to estimate $\Df(P \Vert \wh P_G)$ through the dual variational form in \eqref{eq:dual}, resulting in the following minimax objective \citep{nowozin_f-gan_2016}.
\begin{align}\label{eq:dualtraining}
\min_{G\in \cal G}\max_{\T\in \cal T} \E_{\vx\sim P}\left[T(\vx) \right] - \E_{\vx\sim \wh P_G}\left[ f\s(T(\vx))\right].
\end{align}
The optimization procedure is detailed in Algorithm~\ref{alg:naiveapproach}. 
An important special case is that of the original paper of \cite{goodfellow_generative_2014}, where $D(\vx)\coloneqq\exp\left(T(\vx)\right)$, and the minimax objective is as follows:
\begin{align}
\min_{G}\max_{D} \E_{\vx\sim P}\left[\log\left(D(\vx)\right) \right] + \E_{\vx\sim \wh P_G}\left[ \log\left(1-D(\vx)\right)\right].
\end{align}

\subsection{Rejection Sampling}\label{subsec: rejection sampling}
\label{subsec:rejsam}
Rejection Sampling is a classical method to generate samples from a distribution using samples drawn from a different distribution. In the context of this paper,  samples drawn from $\whP$ are accepted or rejected using an {\em acceptance function} $a: \cX\to [0,1]$, where $a(\vx)$ is the probability of accepting a sample $\vx$ from $\whP$. The distribution induced by the rejection procedure based on $a$ is a new distribution in $\cP(\cX)$ denoted $\widetilde{P}_a$. 
The density $\widetilde p_a(\vx)$ of $\wtP_a$ has the following form:
\begin{align}\label{eq:refineddensity}
    \widetilde p_a(\vx) = \frac{\wh p (\vx)a(\vx)}{Z},
\end{align} 
where $Z>0$ is a normalizing constant that ensures that $\int_{\Xset}\widetilde p_a(\vx)=1$. The overall acceptance rate is $\EE{\whP}{a(\vx)} =  Z$. Note that $Z\leq 1$.
If $p, \wh p$ are known, and if there are no constraints on the sampling budget (i.e., no lower limit on $Z$), then $a$ can be set to $a(\vx) = \frac{p (\vx)}{\wh p (\vx)M}$ with $M = \sup_{x\in \cX}\frac{p(\vx)}{\wh p(\vx)}$ so that $\widetilde{P}_a$ matches perfectly the target distribution $P$ because $\widetilde p_a(\vx) = \whp(\vx) \frac{p (\vx)}{\wh p (\vx)ZM} = p(\vx)$ and we have $Z=1/M$. 
However in practice for high-dimensional $\cX$, $M$ can take high values and set a very low acceptance rate \citep{mackay_information_2005}.


%
\paragraph{Rejection Sampling for GANs:}
\cite{azadi_discriminator_2019} propose \emph{Discriminator Rejection Sampling (DRS)} scheme wherein a trained discriminator $\T$ is used to approximate the likelihood ratio via the formula,
\begin{align}\label{eq:lkhratio}
    r(\vx) = \nabla f\s\left(\T(\vx)\right),
\end{align}
which is an approximation of \eqref{eq:densityestimation}.
Thus, the acceptance function of DRS is given by $a_{\mathrm{DRS}}(x)=\frac{r\left(\vx\right)}{M}$,
where $M=\sup_{\vx}  \left\{r(\vx)\right\}$ is estimated using samples $\vx\sim\whP$. To account for low acceptance rate, DRS uses a hyper parameter $\gamma$ to adjust the acceptance rate as,
\begin{align}
    a_{\mathrm{DRS}}\left(\vx\right)=\frac{r(\vx)}{M}e^{-\gamma}.
\end{align}
In practice, the discriminator $T$ is calibrated such that $\EE{\whP}{r(\vx)} =  1$ which results in an overall acceptance rate of $\EE{\whP}{a(\vx)} =  \frac{e^{-\gamma}}{M}$. A low value of $\gamma$ (typically $\gamma<0$) boosts the acceptance rate. 

\paragraph{Related sampling methods:} The introduction of DRS has lead to the development of numerous sampling methods that are also applicable to GANs, such as MH-GAN \citep{turner_metropolis-hastings_2019}, DDLS \citep{che_your_2021}, DOT \citep{tanaka_discriminator_2019}, and DG$f$low \citep{ansari_refining_2021}, LatentRS \citep{issenhuth_latent_2022} and even for Normalizing Flows \citep{stimper_resampling_2022}. These methods, relying on gradient ascent or the training of a latent model, have showcased their potential through various applications. However, the sampling is computationally expensive and are not as efficient under a limited time constraint.    
\paragraph{Accounting for rejection during training:} While the majority of methods employ the rejection sampling scheme post-training, incorporating an \textit{a priori} perspective on the sampling procedure also yields good results empirically. For example, \citet{grover_variational_2018} and  \citet{stimper_resampling_2022} have embedded latent rejection sampling within their training processes, applying it within a variational inference context and a Normalizing Flow framework, respectively. 

\section{OPTIMAL BUDGETED REJECTION SAMPLING (OBRS)}
\label{sec:optirej}
Rejection sampling exhibits a well-established efficiency on low-dimensional samples; but the acceptance rate drops when it is applied to higher dimensional samples \citep{mackay_information_2005}. In this section, we study the problem of rejection sampling under a limited sampling budget $K \in [1, \infty)$, where $K$ represents the expected  number of samples drawn from $\wh P_G$ required to generate a sample from $\wtP_a$. 
We start by introducing a method to find the optimal acceptance function under a given budget $K$ (thus addressing research Question~\ref{question:q1}), then  we characterize the improvement provided by this new method using  Precision and Recall for generative models \citep{sajjadi_assessing_2018}. 

\subsection{Optimal acceptance function}

We recall that $P$ is the true data distribution, $\whP_G$ (or $\whP$ for short) is the distribution induced by the generator, and $\wtP_a$ is the distribution obtained by applying the acceptance function $a$ on samples from $\whP$. Given a fixed $\whP$, our goal is to find the acceptance function $a$ that minimizes the  divergence between $P$ and $\wtP_a$ under a budget $K$, as follows: 
\begin{align}
\begin{split}\label{eq:divproblem}
        \min_a &\quad\Df(P\Vert \wtP_a) \\
        \mbox{s.t. } & \begin{cases}\EE{\whP}{a(\vx)} \geq 1/K, \\
        \forall \vx \in\Xset, \, 0\leq a(\vx)\leq 1. \end{cases}
\end{split}
\end{align}
Here the constraint $\EE{\whP}{a(\vx)} \geq 1/K$ is used to bound the expected acceptance rate. 
For $K=1$, the only $a$ satisfying the constraints in \eqref{eq:divproblem} is the unit function $a(\vx)=1 \ \forall\ \vx\in \cX$. This case corresponds to no rejection (or accept w.p. $1$) and we have $\wtP_a = \whP$ almost everywhere.

Note that the objective $\Df(P\Vert \wtP_a)$ is continuous with respect to $a$. Since the constraint set for $a$ is closed and bounded, there exists an optimal $a$ for problem~\eqref{eq:divproblem}. In the following theorem, we give an explicit form for the optimal solution $\aobs$ for finite $\cX$ using Lagrangian~duality.


\begin{theorem}[Optimal Acceptance Function]\label{thm:optrej}
For a sampling budget $K\geq 1$ and finite $\cX$, the solution to problem (\ref{eq:divproblem}) is,
\begin{align}
		\aobs(\vx)  = \min\left(\frac{p(\vx)}{\wh p(\vx)}\frac{c_K}{M}, 1\right),
	\end{align}
where $c_K\ge  1$ is such that $\E_{\vx\sim \wh p}[\aobs(\vx)] =1/K$.
\footnote{This acceptance function was previously introduced by \cite{grover_variational_2018}), with the sole argument that it is a "natural" approximation of the optimal acceptance function. No theoretical argument was provided.} 
\end{theorem}
Few observations should be made on Theorem~\ref{thm:optrej}:
\begin{itemize}
    \item The constant $c_K$ is solely determined by $K$. In practice, we can compute it using a dichotomy algorithm (detailed in Appendix~\ref{app:subsec:algoc}).
    \item A budget greater than $M = \sup_{x\in \cX}\frac{p(\vx)}{\wh p(\vx)}$ (unbudgeted sampling) implies that $c_K=1$, and thus $\aobs(\vx)  = \frac{p(\vx)}{M\wh p(\vx)}$. 
    \item The optimal function $\aobs$ does not depend on $f$ meaning that OBRS is optimal for various  \fdivs including ones that are more sensitive to covering the probability mass or ones that are more sensitive capturing modes. This observation is the base of our analysis on how the OBRS improves Precision and Recall in Section~\ref{subsec:OBRSPR}
\end{itemize}

\begin{figure}[t!]
    \centering
    \includegraphics[width=\linewidth]{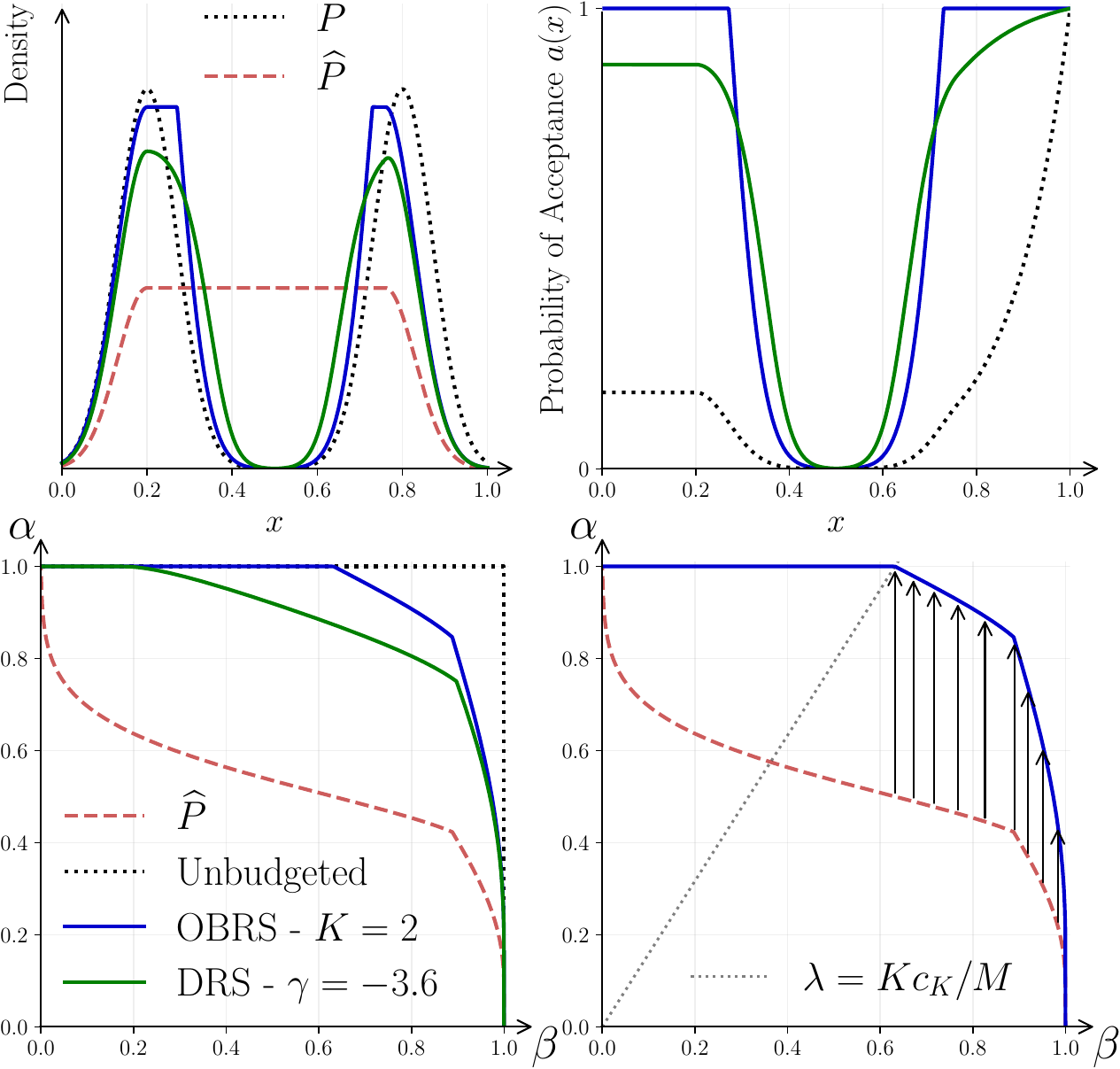}
    \caption{Comparing Unbudgeted, DRS \citep{azadi_discriminator_2019} and OBRS (ours) for a one-dimensional example. 
    DRS and OBRS are tuned to reach an acceptance ratio of $50\%$. TL: The target and learned distributions $P$ and $\whP$, along the refined distributions. TR: The acceptance functions for the unbudgeted rejection sampling (dotted black), OBRS (blue), and the DRS (green). BL: The PR-Curves of the different models. BR: Visualisation of the improvements by the OBRS. The straight dotted line corresponds to $\lambda=Kc_K/M$.
    }
    \label{fig:1D}
\end{figure}

Figure~\ref{fig:1D} illustrates Theorem~\ref{thm:optrej} on a one-dimensional example. On the top-left of Figure~\ref{fig:1D}, we draw $\whP$ and $P$, where the target distribution $P$. On the top-right, DRS (green) and OBRS (blue) are compared.  We can observe how $\aobs$ and $a_{\mathrm{DRS}}$ lead to different refined distributions $\wtP_a$.

Finally, we present a theorem showing how much OBRS reduces the $f$-divergence in general. We show that for any \fdiv, the improvement is linear to $K$. We also give a tighter version of the bound for the Kullback-Leibler Divergence. Proofs for both results are in Appendix~\ref{app:sec:bounds}.
\begin{theorem}\label{thm:bound}
For any $f$-divergence,
we have 
\[
\Df\left(P\Vert\wtP_a\right)\le\Df\left(P\Vert\whP\right)-\min\left(1,\frac{K-1}{M}\right)\Df\left(P\Vert\whP\right)
\]
and for Kullback-Leibler we have for $\gamma=\frac{\log K}{\log M}$
\[
\KL(P\Vert\wtP)\le(1-\gamma)\left(\KL(P\Vert\whP)-\D^\mathrm{R}_{\gamma}(P\Vert\whP)\right)
\]
where $\D^\mathrm{R}_{\gamma}$ is the R\'enyi divergence with parameter $\beta$ 
\end{theorem}

\subsection{Improvement on the Precision/Recall}\label{subsec:OBRSPR}

A number of recent publications have stressed the importance of measuring the quality of generative models using precision and recall  \citep{ kynkaanniemi_improved_2019,djolonga_precision-recall_2020, naeem_reliable_2020, cheema_precision_2023, kim_toppr_2023, verine_precision-recall_2023, bronnec_exploring_2024}. In the context of generative modeling, \emph{precision} measures the quality of the generated samples,  while \emph{recall} which measures the diversity of the samples. In this section, we introduce Theorem~\ref{thm:improvalpha}, that provide a clear characterization of the improvement provided by OBRS in terms of precision and recall. 




To model the set of all precision-recall tradeoffs, \cite{simon_revisiting_2019} introduced the notion of \emph{Precision-Recall Curve} between to distributions $P$ and $\widehat{P}$. This curve, named $\PRd(P,\widehat{P})$,  is composed of all coordinate points $\left(\alpha_\lambda, \beta_\lambda\right)_{\lambda\in[0,+\infty]} \in [0, 1]^2$ defined as follows.
\begin{align}
\begin{cases}
        \alpha_\lambda = \E_{\whP}\left[\min\left\{\lambda \frac{p(\vx)}{\whp(\vx)}, 1\right\}\right] \\[3pt]
        \beta_\lambda =\E_{P}\left[\min\left\{1,\frac{\whp(\vx)}{p(\vx)}\frac{1}{\lambda}\right\}\right]
\end{cases}
\end{align}
Intuitively, if  $(\alpha,\beta)$ belongs to the Precision-Recall curve, this means that for some fixed recall $\beta$, the best achievable precision is $\alpha$. A more comprehensive definition and explanation of Precision/Recall for generative models is given in Appendix~\ref{app:sec:PR}.

\begin{theorem}[Precision and Recall Improvement]\label{thm:improvalpha}
Let $K\leq M$ be the budget for the OBRS detailed  in Theorem~\ref{thm:optrej}. For any $(\alpha,\beta)\in \PRd(P,\widehat{P})$ we have $\left(\alpha',\beta\right)\in \PRd(P,\wtP_{\aobs})$ with $\alpha'=\min\left\{1,K\alpha\right\}$.
\end{theorem}

This theorem shows that for any fixed recall, OBRS consistently improves precision. More precisely,  the improved PR-curve is a $K$-fold vertical scaling of the intial PR-curve capped to $1$. The bottom-right part of Figure \ref{fig:1D} illustrates this phenomenon.
In Appendix~\ref{app:sec:djolonga}, we show a similar theorem for another popular precision-recall measure called the \emph{Information Divergence Frontier} \citep{djolonga_precision-recall_2020}.

\section{TRAINING WITH OBRS}\label{sec:trainobrs}
In traditional GAN training, the generator $G$ is optimized without considering any \textit{a priori} knowledge regarding the rejection sampling that occurs post-training, potentially leading to suboptimal generative models. This section advocates training with OBRS (Tw/OBRS) for GANs models. First, we introduce the theoretical improvements and the observed effects on the loss function.  Then, we introduce an algorithm to incoporate OBRS in the training procedure.

\subsection{Principle of Training with OBRS}\label{subsec: TOBRS principle}

Let us reformulate Rejection Sampling in the domain of probability measures. Define $B_K(\whP)=\left\{\wtP\in \cal P(\cX) | \D^\mathrm{R}_{\infty}(\wtP\Vert \whP)\leq  \log K \right\}$, where $\D^\mathrm{R}_{\infty}(\wtP\Vert\whP)= \log(\sup_{\vx}\left\{\wtp(\vx)/\whp(\vx)\right\})$ denotes the max-divergence (a limiting case of the $\alpha$-Rényi Divergence $\Da^R$ with $\alpha\to\infty$).
Note that $B_K(\whP)$ is a convex set. Moreover, the following inclusion holds for any $K_2\geq K_1\geq 1$.
\begin{align}
    B_{K_1}(\whP) \subseteq B_{K_2}(\whP).
\end{align}
The following lemma shows that $B_K(\whP)$ characterizes the set of distributions allowed by a budgeted rejection sampling procedure.
\begin{lemma}\label{lem: ball}
 $\wtP\in B_K(\whP)$ if and only if there exist an acceptance function $a:\setX\to[0, 1]$, and a normalization constant $Z$ such that $\wtp(\vx)=\whp(\vx)a(\vx)/Z$ and the acceptance rate is greater that $1/K$.
\end{lemma}
Consider $\wh{\cal P}=\left\{\whP=G_\# Q \vert   G\in \cal G\right\}$, the set of  all distributions $\whP$ induced by the generator functions from a fixed latent distribution $Q$. 
By separating the training process from the rejection sampling process, we are, in effect, solving a two-step minimization problem given below.
\begin{align}
    \text{First solve }\whP\opt\in \argmin_{\whP\in \whcP} \Df (P\Vert \whP);\label{eq: obj train}\\
    \text{Next solve }\wtP\opt\in \argmin_{\wtP\in B_K(\whP\opt)} \Df(P\Vert \wtP).\label{eq: obj RS}
\end{align}
Crucially,  $\whP\opt$ is chosen by the training procedure to optimize \eqref{eq: obj train} whereas the final output distribution $\wtP\opt$ is assessed via \eqref{eq: obj RS}, resulting in a mismatched objective. By incorporating the rejection scheme into the training objective, we get:
\begin{align}\label{eq:trainobrs_prequel}
 \min_{\whP\in\wh{\cal P}}\min_{\wtP\in B_{K}(\whP)}\Df(P\Vert\wtP).
\end{align}


\begin{figure}[t!]
\subfloat[The loss $\GAN(P\Vert\wtP)$ is calculated for every parameter $\theta$ for different budgets $K$. For $K=1$, it is $\GAN(P\Vert\whP)$. \label{fig:DsmoothDiv}]{ \includegraphics[width=\linewidth]{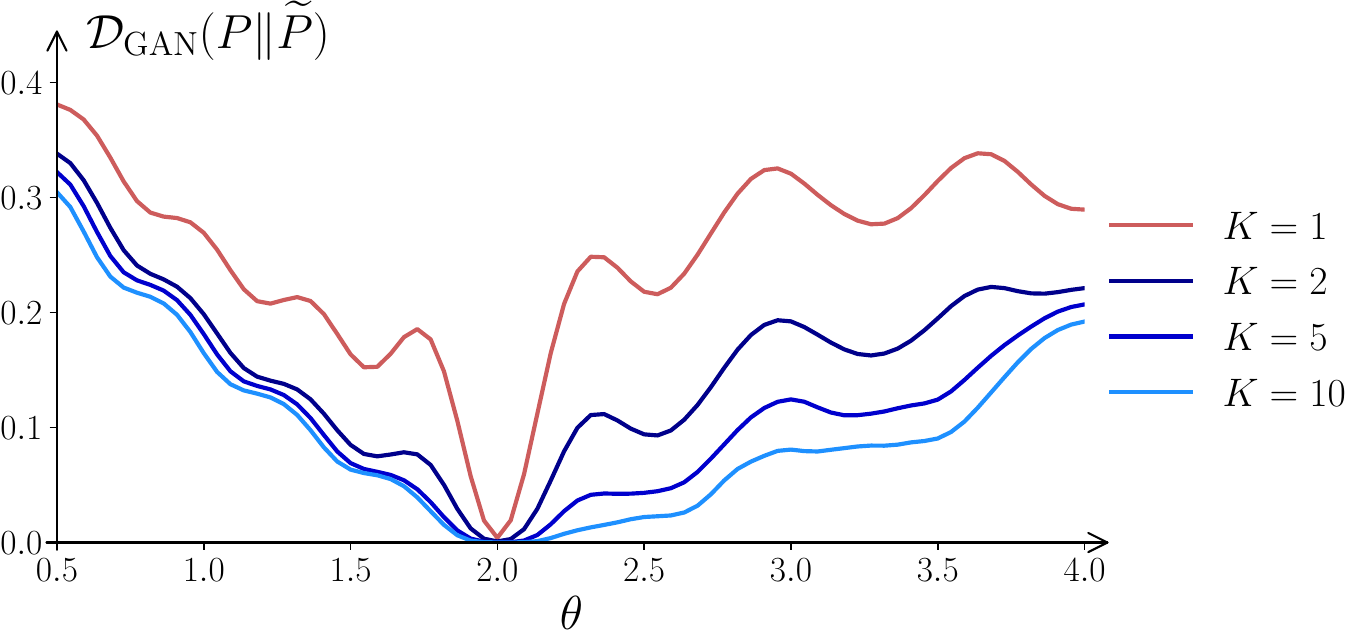}}
\vspace{10pt}
\subfloat[{The target distribution $P$ (in dotted black) is a mixture of 10 Gaussians with $\sigma^2=0.3$. The approximate distribution is a mixture of 10 Gaussians of $\sigma^2=0.4$ separated by $\theta$. $\wtP$ is computed with OBRS and a budget of $K=2$. Densities are re-scaled and cropped to $\left[-7, 7\right]$ for readability.  \label{fig:DsmoothP}}]{ \includegraphics[width=\linewidth]{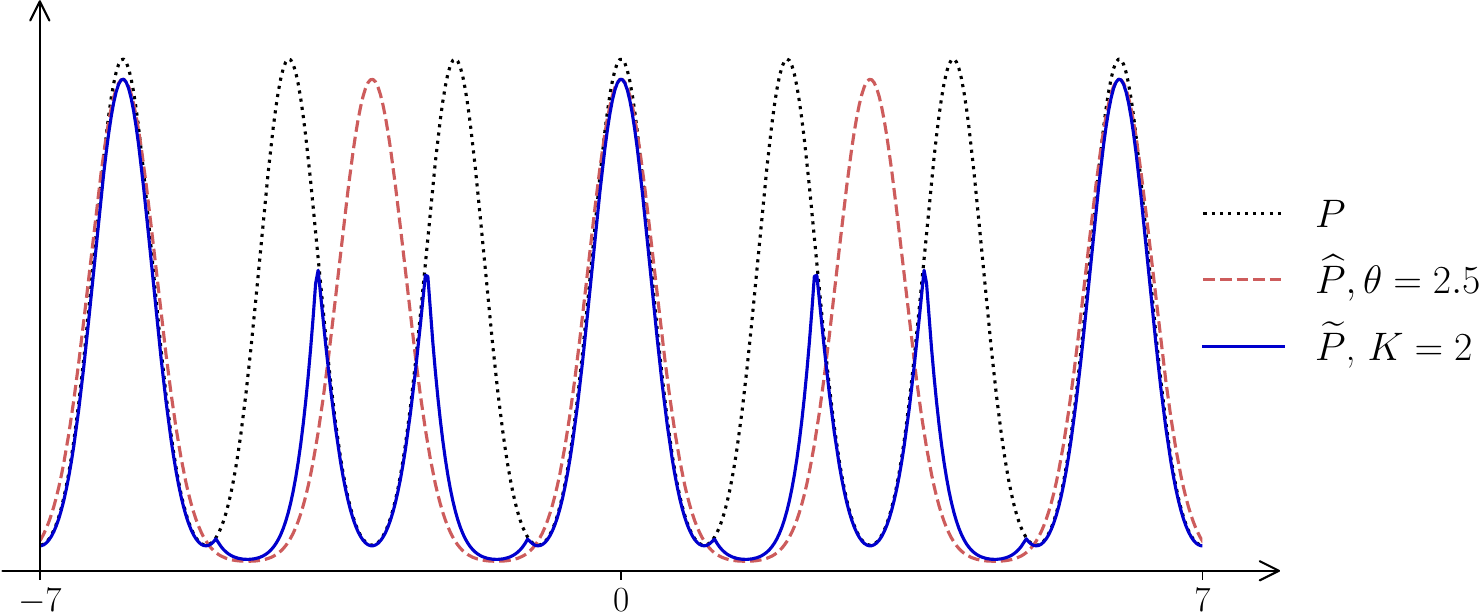}}
   \caption{The loss $\GAN(P\Vert\wtP)$ is flatten by the OBRS scheme. As the budget $K$ increases, the number of local minima decreases. } \label{fig:Smooth}
\end{figure}
This re-framing of the objective has the following advantages.

\begin{figure}[t!]
\subfloat[$\GAN$ is calculated between $P$ and $\whP$ (left) or $\wtP$ (right), for all parameters $(\mu, \sigma)$. The stars ($\filledstar$) highlight the minima.\label{fig:1DGaussiansDivs}]{\includegraphics[width=1\linewidth]{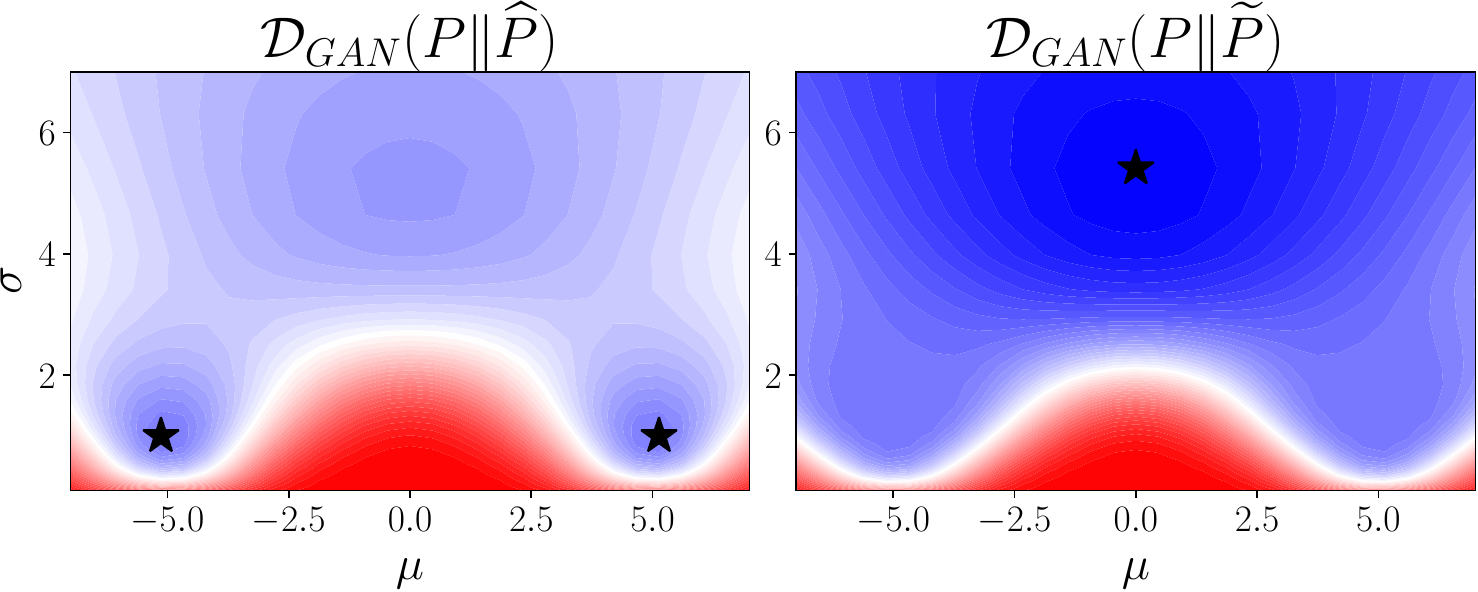}}
\vspace{10pt}
\subfloat[For the target $P$ (in dotted black), the approximation $\whP$ (in dashed red) corresponds to a minima in Fig.~\ref{fig:1DGaussiansDivs}. The post-OBRS distribution $\wtP$ (in solid blue) is for $K=2$.]{\includegraphics[width=1\linewidth]{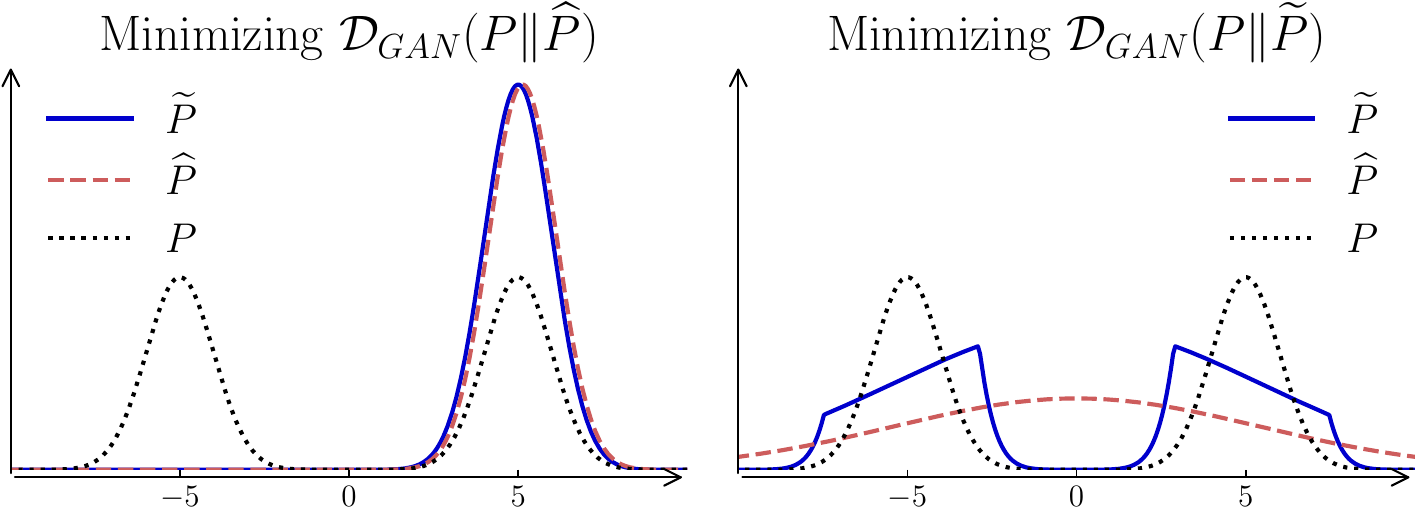}}
\caption{One dimensional example of $\Df$ minimization: $P$, a mixture of two gaussians is approximated by Gaussian  $\whP=\cal N(\mu, \sigma^2)$. The distribution $\whP$ that minimizes $\GAN(P\Vert\wtP)$ leads to a drastically better approximation $\wtP$ of $P$ than the post rejection distribution induced by the $\whP$ that minimizes $\GAN(P\Vert\whP)$.  }
    \label{fig:1DoptimalG}
\end{figure}
\begin{figure*}[b!]
\vspace{-10pt}
   \begin{minipage}[b]{0.5 \textwidth}
\begin{algorithm}[H]
\caption{Traditional GAN training procedure}
\label{alg:naiveapproach}
\begin{small}
\begin{algorithmic}
\Repeat
  \State Update $T$ by ascending the gradient of $$
  \E_{\vx\sim P}\left[T(\vx) \right] - \E_{\vx\sim \wh P_G}\left[ f^*(T(\vx))\right].$$
  \State Update $G$ by descending the gradient of $$
   - \E_{\vx\sim \wh P_G}\left[ f^*(T(\vx))\right].$$
\Until{convergence.}
\end{algorithmic}
\end{small}
\end{algorithm}
   \end{minipage}
      \begin{minipage}[b]{0.5 \textwidth}
       \begin{algorithm}[H]
\caption{GAN Tw/OBRS}
\label{alg:TOBRS}
\begin{small}    
\begin{algorithmic} 
\Repeat
  \State Update $T$ by ascending the gradient of $$\E_{\vx\sim P}\left[T(\vx) \right] - \E_{\vx\sim \wh P_G}\left[ f^*(T(\vx))\right].$$
  \State Update $c_K$ such that $\EE{\whP_G}{\aobs(\vx)} \leq 1/K$. \\
  \hspace{2cm}(See Alg \ref{alg:searchc} in App \ref{app:subsec:algoc}) for details.)
  \State Update $G$ by descending the gradient of $$
   \E_{\vx\sim \wh P_G}\left[K\aobs(\vx) f\left(\frac{r\left(\vx\right)}{K\aobs(\vx)}\right)\right].$$
\Until{convergence.}
\end{algorithmic}
\end{small}
\end{algorithm}
   \end{minipage}
\end{figure*}

\paragraph{Flattening effect on the parameter landscape:}
Note that the objective in \eqref{eq:trainobrs_prequel} can be written as, 
\begin{align}\label{eq:trainobrs}
    \min_{\wtP\in \bigcup_{\whP \in \wh{\cal P}}B_K(\whP)}  \Df(P\Vert \wtP).
\end{align}
Observe that the domain of $\wtP$ is the dilatation of $\whcP$ by the convex set $B_K$, resulting in a smoother set $\bigcup_{\whP \in \wh{\cal P}}B_K(\whP)$. In practice, this results in a flattened loss landscape for optimizing over $\whP$ as in \eqref{eq:trainobrs_prequel}, thus preventing the model from getting stuck in suboptimal local minima. This concept is demonstrated with two examples, showcasing its ability to flatten the parameter landscape. Firstly, Figure~\ref{fig:Smooth} shows a one-dimensional example where the loss is flattened by OBRS. Secondly, Figure~\ref{fig:MNISTSmooth} illustrates a GAN trained to generate MNIST samples. Like in the approach of \cite{li_visualizing_2018}, we present the loss in two arbitrary directions of the parameter space. We observe that OBRS not only reduces the loss but also flattens the landscape, thereby aiding in avoiding local minima. More details are provided in Appendix~\ref{app:sec:MNIST}.


\paragraph{A mass-covering $\whP$:}
The optimal $\whP$ might be different between \eqref{eq: obj train} and \eqref{eq:trainobrs_prequel}. Theorem~\ref{thm:improvalpha} explicitly states that OBRS is more efficient on mass-covering models rather than mode-seeking ones, as it improves precision. Taking the rejection sampling into account in the training procedure is pushing the distribution $\whP$ to be more \emph{suitable} for reject, and thus: more-mode covering.  
For instance, consider a target distribution $P$ as the Gaussian mixture presented in Figure~\ref{fig:1DoptimalG}. Assume that the expressivity of $\whP$ is limited to a single Gaussian $\cal N(\mu, \sigma)$. If the goal is to naively minimize $\GAN$ (defined in Table~\ref{tab:fdiv}), then, because of the mode-covering property of the divergence, $\whP$ covers only one mode. In that case, Theorem~\ref{thm:improvalpha} shows that only the precision can be improved, and thus a limited-budget rejection sampling scheme will not reshape $\whP$, leading to poor coverage.  While, if $\mu$ and $\sigma$ are set to directly minimize $\GAN(P\Vert \wtP)$, then the distribution $\whP$ changes drastically into a mass covering distribution, allowing the rejection process to match more closely (in terms of $\GAN$).

\begin{table*}[!t]
 \caption{Mixture of 25 Gaussiansin 2D. Metrics for the different sampling Methods: Recall ($\uparrow$) and Precision ($\uparrow$) as defined in \cite{dumoulin_adversarially_2017};  Calls ($\downarrow$) of $G$ and $D$ are the number of times the models are called to generate $2500$ samples; Time ($\downarrow$) is the time required to generated $2500$ samples.
 For every metrics, we give the average and standard deviation for $1000$ generations of $2500$ samples. Best results are emphasized in \textbf{bold}.} \label{tab:2Dexp}
    \begin{tabularx}{\textwidth}{X|ccccc}
        Model & Recall ($\%$) & Precision ($\%$) & Call of $G$  & Call of $D$ & Time (s)\\\hline
Baseline $G$ & $100.0 \pm 0.0$ & $55.80 \pm 0.99$ & $2500 \pm 0$ & $0 \pm 0$ & $0.03 \pm 0.01$ \\  \hline\addlinespace[0.1em]
OBRS (ours) ($K = 2.6$) & $100.0 \pm 0.0$ & $\mathbf{92.54 \pm 0.54}$ & $6262 \pm 92$ & $\mathbf{6262 \pm 92}$ & $\mathbf{0.45 \pm 0.01}$ \\
DRS ($\gamma = -0.9$) & $100.0 \pm 0.0$ & $89.87 \pm 0.59$ & $6411 \pm 93$ & $6411 \pm 93$ & $\mathbf{0.46 \pm 0.01}$ \\
MH-GAN ($n_{\mathrm{ite}}=2$) & $100.0 \pm 0.0$ & $89.98 \pm 0.61$ & $6415 \pm 45$ & $19292 \pm 23$ & $6.38 \pm 0.09$ \\
DOT ($n_{\mathrm{ite}}=3$) & $100.0 \pm 0.0$ & $58.47 \pm 1.00$ & $\mathbf{2500 \pm 0}$ & $7500 \pm 0$ & $0.94 \pm 0.14$ \\
DG$f$low ($n_{\mathrm{ite}}=3$) & $94.81 \pm 2.83$ & $56.00 \pm 1.02$ & $7500 \pm 0$ & $7500 \pm 0$ & $0.67 \pm 0.13$ \\ \hline
    \end{tabularx}
\end{table*}

 \subsection{Implementing Tw/OBRS}
To implement Tw/OBRS i.e., to solve for the combined objective in \eqref{eq:trainobrs_prequel}, we need samples from $\wtP$ in order to evaluate the final loss $\Df(P\Vert \wtP)$.
One direct approach is to train a discriminator $\widetilde T$ to estimate $D(P\Vert \wtP)$, and then training a generator to minimize the estimate by minimizing:
\begin{align}
    -\E_{\wtP}\left[f\s(\widetilde T (\vx))\right]=-\E_{\whP}\left[Ka_O(\vx)f\s(\widetilde T (\vx))\right].
\end{align}
But, this would require to compute $\aobs$ which depends on $r(\vx)$ that is obtained by training a discriminator $T$ on $D(P\Vert \whP)$. In other words, it would require two discriminators $T$ and $\widetilde T$. Instead, we propose a method that would require training only a single discriminator $T$ and leverage the primal form of \fdiv give in \eqref{eq:fdiv} to estimate $\Df(P\Vert \wtP)$ as follows. 
\begin{align*}
        \Df(P\Vert \wtP) 
        &= \E_{\whP}\left [ f\left( \frac{p(\vx)}{\widetilde{p}(\vx)}\right)  \right]\\
        &= \E_{\wtP}\left [ \frac{\widetilde{p}(\vx)}{\whp(\vx)} f\left( \frac{p(\vx)}{\widehat{p}(\vx)}\frac{\whp(\vx)}{\widetilde{p}(\vx)}\right)  \right]\\
        &= \E_{\whP}\left [ Ka_O(\vx) f\left(\frac{\nabla f\s\left(T(\vx)\right)}{Ka_O(\vx)}\right)\right],
\end{align*}
where the last equality follows by plugging in the likelihood ratio estimate of \eqref{eq:densityestimation}.
We propose  Algorithm~\ref{alg:TOBRS} that trains a model $G$ to minimize the estimated \fdiv between $P$ and $\wtP$. This algorithm is, in terms of algorithmic complexity, equivalent to the traditional GAN training procedure detailed in Algorithm~\ref{alg:naiveapproach}. We detail in Appendix~\ref{app:subsec:complexity} how the update of $c_K$ affects the time of the training procedure.

\section{EXPERIMENTAL RESULTS}\label{sec:xp}

\subsection{Sampling methods for 25 Gaussians}
We first evaluate our methods  using a grid of $5\times 5$ two-dimensional Gaussians following the experimental protocol used by the authors of other GANs sampling methods \citep{azadi_discriminator_2019, turner_metropolis-hastings_2019, ansari_refining_2021, che_your_2021, tanaka_discriminator_2019}. Hyperparameters of every methods are set to achieve about 40\% acceptance rates ($K=2.6$) in order to obtain comparable performances. We then measure precision and recall using the methodology proposed by \citet{dumoulin_adversarially_2017} as well as execution time for every method. Results are presented in Table~\ref{tab:2Dexp}. We observe almost every method achieve 100\% recall but that OBRS outperforms all other methods in terms of both precision and sampling time. Detailed experimental settings and a discussion how budget and time affect the performances are available in Appendix~\ref{app:subsec:xp2D}. We further demonstrate the impact of the distribution $\whP$ and particularly the influence of $M$ on the performance disparity between OBRS and DRS. In our experiments, we select hyperparameters to achieve similar acceptance rates. Yet, for varying budgets, the difference of efficiency between OBRS and DRS may increase. Figure~\ref{app:fig:obrsvsdrs} illustrates the behavior of these methods for different values of $M$. 

\begin{figure}[t!]
    \centering
    \includegraphics[width=\linewidth]{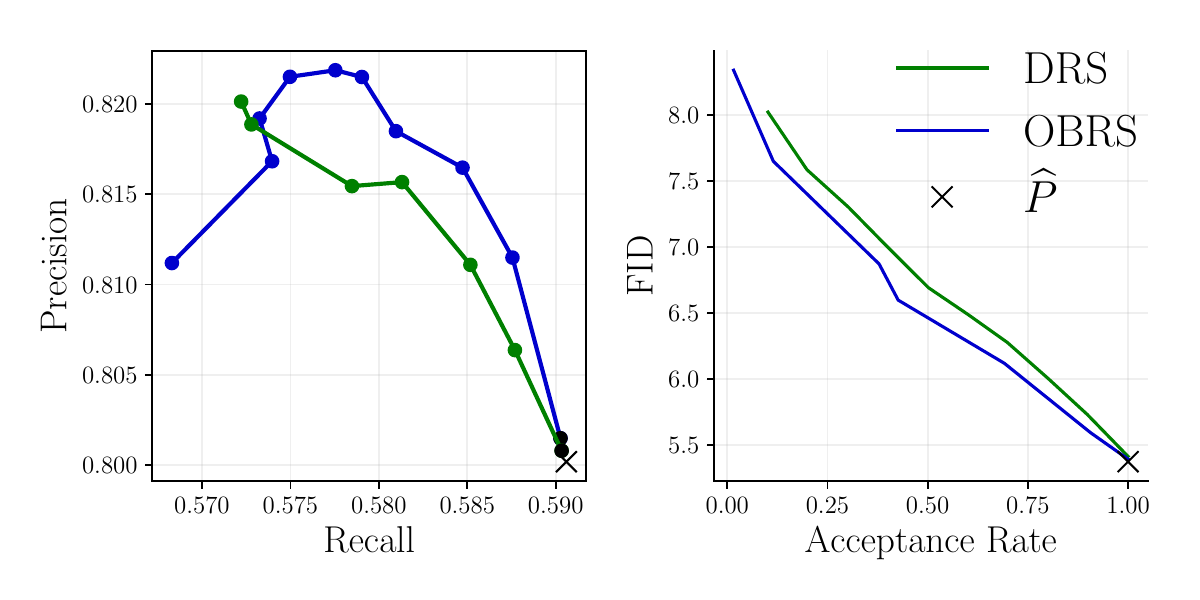}
    \caption{DRS vs. OBRS on a pre-trained BigGAN on CelebA.  GAN Baseline model $\whP_G$, Post-rejection distribution $\wtP_{a_{\mathrm{DRS}}}$ with DRS, Post-rejection distribution $\wtP_{\aobs}$ with OBRS. (Left) Precision and Recall for different budgets. Lowest budget in black. (Right) FID as a function of the acceptance rate.}
    \label{fig:DRSvsOBRS}
\end{figure}

\begin{table}[b!]
\centering
     \caption{OBRS applied on a Diffusion Model EDM \citep{karras_elucidating_2022} with a classifier trained by \citet{kim_refining_2023}. We observe no relevant improvement on the Recall, a slight improvement on the Precision and a significant improvement on the FID.} \label{tab:cifaredm}
     \begin{tabular}{c|ccc}
        Acceptance rate & FID & P & R   \\\hline
        0.25 & $1.57$ & $78.48$  &  $86.73$  \\
        0.50 & $1.58$ &  $78.23$  & $86.05$ \\
        0.75 & $1.77$ & $77.94$  & $86.54$  \\
        1 & $1.97$ & $77.91$ & $86.62$     \\\hline
    \end{tabular}
\end{table}

\subsection{OBRS for a pre-trained model}

We now investigate how OBRS performs in high dimension. To do so, we use a BigGAN model  \citep{brock_large_2019} pre-trained on CelebA. Note that the model is originally  trained with the hinge loss which is saturating according to \cite{azadi_discriminator_2019} and leads to a discriminator that is not suitable for density estimation. Thus, following their recommendations, we fine-tune the discriminator to improve density estimation. In Figure~\ref{fig:DRSvsOBRS}, we evaluate the resulting model in terms of  Precision and Recall \citep{kynkaanniemi_improved_2019} for 10k samples for $k=5$ and the FID for 50k samples. When evaluating for multiple budgets between $1$ and $M$, we observe that OBRS outperforms DRS in terms of FID and, for acceptance rates greater than $30\%$, in terms of precision. We also test the rejection procedure on a diffusion model on CIFAR-10 trained by \citet{karras_elucidating_2022} with a discriminator trained by \citet{kim_refining_2023}. In Table~\ref{tab:cifaredm}, that the OBRS method improves the FID by a significant margin, while the precision is slightly improved and the recall remains stable.

\begin{figure}[t!]
    \subfloat[FID during training. ]{\includegraphics[width=\linewidth]{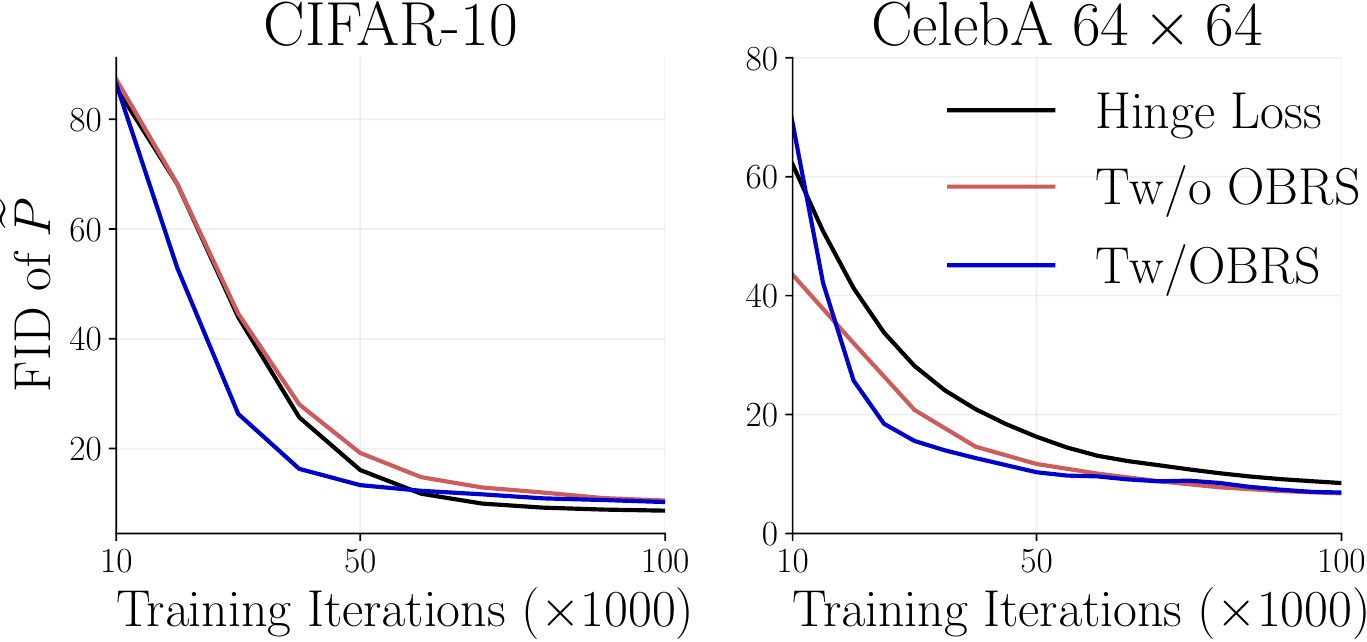}}
    \vspace{10pt}
    \subfloat[Metrics at convergence all between]{
    \begin{tabular}{cc|ccc}
       Dataset & Method & FID & P & R \\\hline  
       CIFAR-10   & Hinge Loss & $\mathbf{8.43}$ & $\mathbf{84.50}$ & $65.39$\\
        $32 \times 32$ & Tw/oOBRS & $11.18$ & $83.24$ & $68.44$ \\
        & Tw/OBRS & $8.98$ & $80.09$ & $\mathbf{69.63}$ \\\hline
       CelebA  & Hinge Loss & $9.33$ & $\mathbf{80.23}$ & $57.78$\\
       $64\times 64$ & Tw/oOBRS & $6.33$ & $78.28$ & $\mathbf{61.02}$ \\
       & Tw/OBRS & $\mathbf{5.42}$ & $78.01$ & $60.29$ \\\hline
    \end{tabular}}
    \caption{Training w/OBRS. We use a BigGAN \citep{brock_large_2019}  trained with hinge loss as a baseline compared to $\GAN$ trained without (Tw/oOBRS) and with (Tw/OBRS) OBRS. All metrics are calculated between $P$ and $\wtP$ with a budget of $K=2$.}
    \label{fig:trainingobrs}
\end{figure}
\subsection{Training with OBRS}
We investigate  the  Tw/OBRS method discussed in Section~\ref{sec:trainobrs}. We use BigGAN and trained in 3 ways: (1) hinge loss (baseline), (2) $\GAN$ loss using the standard method from Algorithm~\ref{alg:naiveapproach} (Tw/oOBRS), and (3) $\GAN$ loss using our new method from Algorithm~\ref{alg:TOBRS} (Tw/OBRS), with $K=2$.
We tested these methods on the  CIFAR-10 and CelebA datasets and showed the results in Figure~\ref{fig:trainingobrs}. To be fair, we evaluate all 3 models on the refined distribution $\wtP$ with a budget of $K=2$. In our experiments, our method demonstrates accelerated convergence and superior performance in terms of FID compared to the alternative approaches. While there is a notable increase in Recall, there is a slight trade-off in Precision.

\begin{table}[t!]
    \centering
    \caption{Fine-tuning with Tw/OBRS. Pre-trained BigGAN  fine-tuned on the $\GAN$ with OBRS. We use a BigGAN trained with the hinge loss as a baseline. All metrics are calculated between the target distribution $P$ and the post-distribution with a budget of $K=2$. }
    \label{tab:finetune}
    \begin{tabular}{cc|ccc}
       Dataset & Method & FID & P & R \\\hline
       CelebA  & Hinge Loss & $9.33$ & $\mathbf{80.23}$ & $57
       .78$\\
       $64\times 64$ & w/OBRS & $\mathbf{3.74}$ & $74.40$ & $\mathbf{65.15}$ \\\hline
        ImageNet   & Hinge Loss & $12.18$ & $\mathbf{27.75}$ & $34.33$\\
        $128 \times 128$ & w/OBRS & $\mathbf{11.65}$ & $26.84$ & $\mathbf{46.16}$ \\\hline
    \end{tabular}
\end{table}

We also fine-tuned models trained on the hinge loss using our method. We used BigGAN models trained on CelebA and ImageNet in Table~\ref{tab:finetune}.  

This set of experiments on training models accounting for rejection shows that intuitions presented in Section~\ref{sec:trainobrs} are confirmed empirically: the models converge faster and leads to an optimal $G$ more mass-covering.

\section{CONCLUSION AND FUTURE WORKS}
In this paper, we introduce the concept of budgeted rejection sampling and go a step further by presenting an optimal acceptance function for this sampling method. We use this method to improve discriminator-based models. However, we believe that our Tw/OBRS scheme can be applied to a broader class of generative models. For instance, one could use our approach for Normalizing Flows using the Learned Acceptance/Rejection  Sampling method of \cite{stimper_resampling_2022}. For diffusion models, there is much greater flexibility to refine the distribution through rejection sampling because one can choose to accept a sample at any iteration of the diffusion process. 
Building on this, one might modify the discriminator refined scored-based sampling of \cite{kim_refining_2023} to improve diffusion models.  

Our work emphasizes the importance of incorporating rejection during the training phase. Practically, this inclusion results in generating distributions with greater recall, ensuring that rejection sampling becomes more effective. 
In Subsection~\ref{subsec: TOBRS principle}, we hypothesize that this improvement may be due to the dilation of the possible set of output distributions
$\whcP$ by a convex set $B_K(\whP)$ during rejection. It would be interesting to further analyze this phenomenon through a theoretical lens. 

\acknowledgments{We are grateful for the grant of access to computing resources at the IDRIS Jean Zay cluster under
allocations No. AD011011296 and No. AD011014053 made by GENCI.}
\bibliographystyle{apalike}
\bibliography{references}
\onecolumn
\aistatstitle{Optimal Budgeted Rejection Sampling for Generative Models\\
Supplementary Materials}
\counterwithin{figure}{section}
\counterwithin{algorithm}{section}
\counterwithin{table}{section}
\appendix
\renewcommand{\thetable}{\thesection.\arabic{table}}
\setcounter{table}{0}
\renewcommand{\thefigure}{\thesection.\arabic{figure}}
\setcounter{figure}{0}
\section{Precision and Recall for Generative Models}\label{app:sec:PR}

According to \cite{sajjadi_assessing_2018}, \emph{The key intuition is that precision should measure how much of $\whP$ can be generated by a “part” of $P$ while recall should measure how much of $P$ can be generated by a “part” of $\whP$.}
In this paper, we evaluate how Optimal Budgeted Rejection Sampling affects a given model. To evaluate the improvement theoretically, we need a mathematically grounded method of assessing models and we need this method to assess quality and diversity independently. To do so, we leverage the notion of \emph{PR-Curves} introduced by \cite{sajjadi_assessing_2018} and revisited for continuous distributions by \cite{simon_revisiting_2019}.

\subsection{From the discrete to the continuous case}

\begin{definition}[Precision and Recall - \citep{sajjadi_assessing_2018}]
For $\alpha, \beta \in[0,1]$,  the probability distribution $\whP$ has a precision $\alpha$ at recall $\beta$ w.r.t. $P$ if there exist distributions $\mu$, $\nu_P$ and $\nu_{\whP}$ such that
$$
P = \beta \mu +(1-\beta)\nu_P \quad \mbox{and}\quad \whP = \alpha \mu + (1-\alpha )\nu_{\whP}
$$
The component $\nu_P$ denotes the part of $P$ that is “missed” by $\whP$. Similarly, $\nu_{\whP}$ denotes the noise part of $\whP$.\quad
\end{definition}

With this definition, the authors define the set of possible precision-recall pairs: $\PR(P, \whP)$. The frontier of  the set of $\PR(P, \whP)$, is the PR-Curve denoted $\PRd(P, \whP)$, parameterized by $\lambda\in[0, \infty]$ and can be computed with the functions:
\begin{align*}
    \alpha(\lambda) = \sum_{\vx_i\in\Xset} \min\left(\lambda p(\vx_i), \whp(\vx_i) \right) \quad \mbox{and} \quad  \beta(\lambda) =  \sum_{\vx_i\in\Xset} \min\left( p(\vx_i), \whp(\vx_i)/\lambda \right) 
\end{align*}

\begin{figure}[H]

    \subfloat[High Precision Example]{\includegraphics[width=0.45\textwidth]{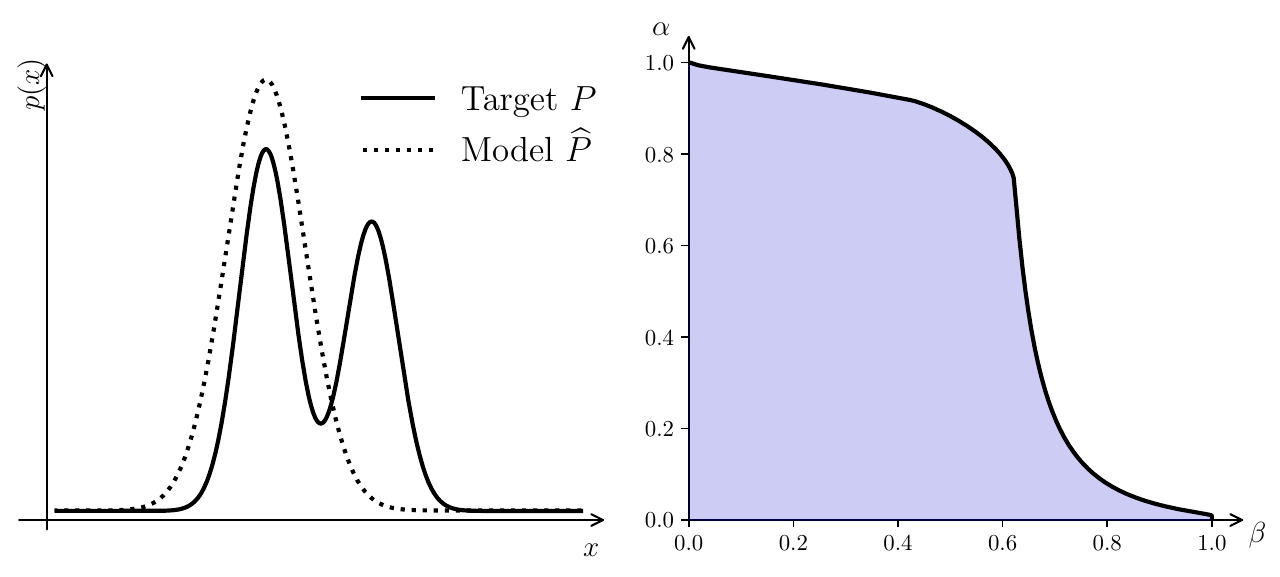}}
 \hfill
    \subfloat[High Recall Example]{ \includegraphics[width=0.45\textwidth]{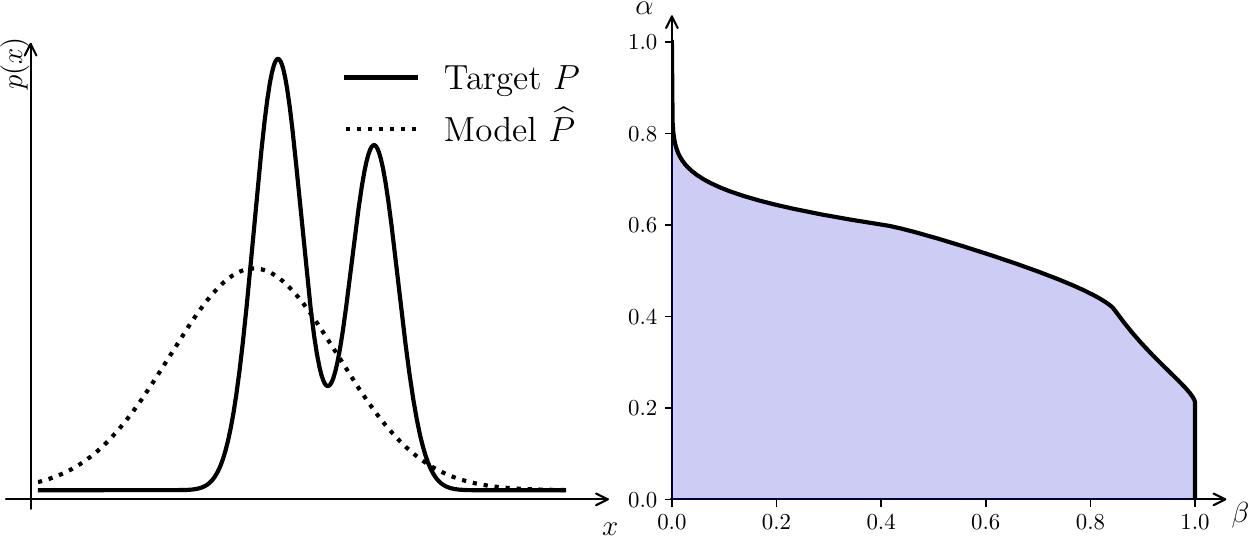}}
\caption{Low dimensional examples of distributions with high recall and limited precision and vice versa, with their corresponding PR-Curves. The colored area is the set $\PR(P, \whP)$ and the solid line in black in the frontier $\PRd(P, \whP)$.}
\end{figure}

\newpage 
This definition has been extended to continuous distributions.
\begin{definition}[Precision and Recall - \citep{simon_revisiting_2019}]
For $\alpha, \beta \in[0,1]$,  the probability distribution $\whP$ has a precision $\alpha$ at recall $\beta$ w.r.t. $P$ if there exists a distribution $\mu$ such that
$$
P \geq \beta \mu  \quad \mbox{and}\quad \whP  \geq \alpha \mu. $$
\end{definition}
If also defines a set $\PR$ and its  frontier is very similar:
\begin{align*}
    \alpha(\lambda) = \int_\Xset \min\left(\lambda p(\vx), \whp(\vx)\right)\d \vx \et \beta(\lambda) = \int_\Xset \min\left( p(\vx), \whp(\vx)/\lambda\right)\d \vx.
\end{align*}
We can reformulate the expressions of the frontier:
\begin{align}
\begin{cases}
        \alpha_\lambda = \E_{\whP}\left[\min\left\{\lambda \frac{p(\vx)}{\whp(\vx)}, 1\right\}\right] \\[3pt]
        \beta_\lambda =\E_{P}\left[\min\left\{1,\frac{\whp(\vx)}{p(\vx)}\frac{1}{\lambda}\right\}\right]
\end{cases}
\end{align}
We can interpret this expression similar to the AUC curve in classification tasks. Consider that the maximum precision and recall are one. Therefore, whenever a point is sampled from $\whP$ such that $\lambda p (\vx)< \whp(\vx)$, the precision decreases further away than $1$. In other terms, all the $\vx$ for which the $\whP$ overestimate $P$ decrease the precision. On the side,  whenever a point is sampled from $P$ such that $\whp(\vx) < \lambda p (\vx)$, the recall decreases further away than $1$, corresponding to the points where $\whP$ underestimates $P$. Let us consider two examples in Figure~\ref{app:fig:goodrecall}~and~\ref{app:fig:goodprecision}.

\begin{figure}[H]
\begin{minipage}[c]{0.3\textwidth}
    \begin{figure}[H]
        \centering
\includegraphics[width=\textwidth]{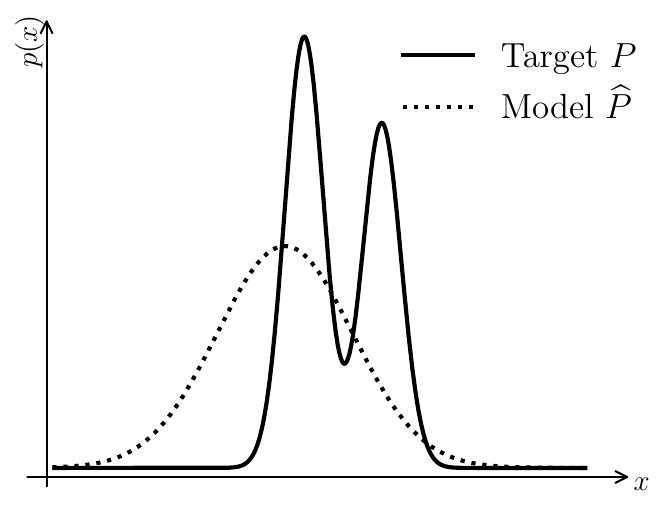}
        \caption{A target distribution $P$ and the approximated distribution $\whP$. In this setup, the model is expected to have a decent recall since it covers $P$ but a poor precision since almost half the weight of $\whP$ does not cover $P$.  In Figures~\ref{app:fig:goodrecalllowl}~and~\ref{app:fig:goodrecallhighl}, we show the PR-Cruve and how it is computed.}
        \label{app:fig:goodrecall}
    \end{figure}
\end{minipage}
\hfill
\begin{minipage}[c]{0.65\textwidth}
    \begin{figure}[H]
        \subfloat[PR-Curve for the model in Figure~\ref{app:fig:goodrecall}, explained for a low $\lambda$. The area in red is $\alpha_\lambda$ and the area in blue is $\beta_\lambda$.]{\includegraphics[width=\textwidth]{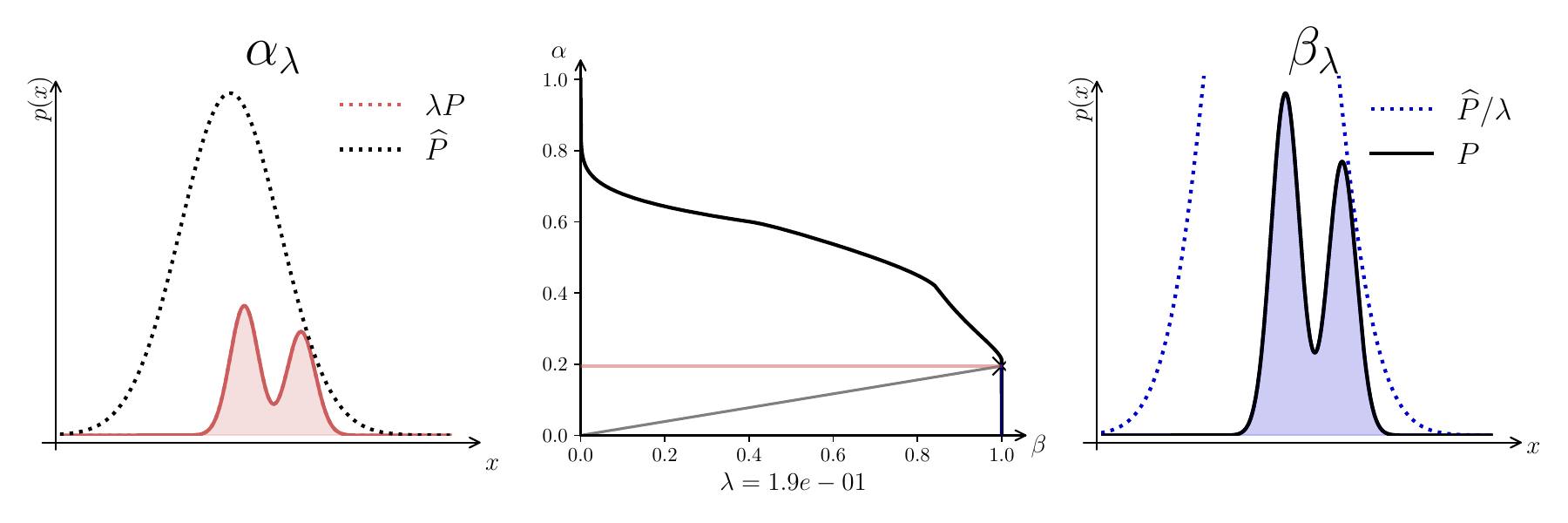}        \label{app:fig:goodrecalllowl}}       
        
        \subfloat[PR-Curve for the model in Figure~\ref{app:fig:goodrecall}, explained for a high $\lambda$. The area in red in $\alpha_\lambda$ and the area in blue is $\beta_\lambda$.]{\includegraphics[width=\textwidth]{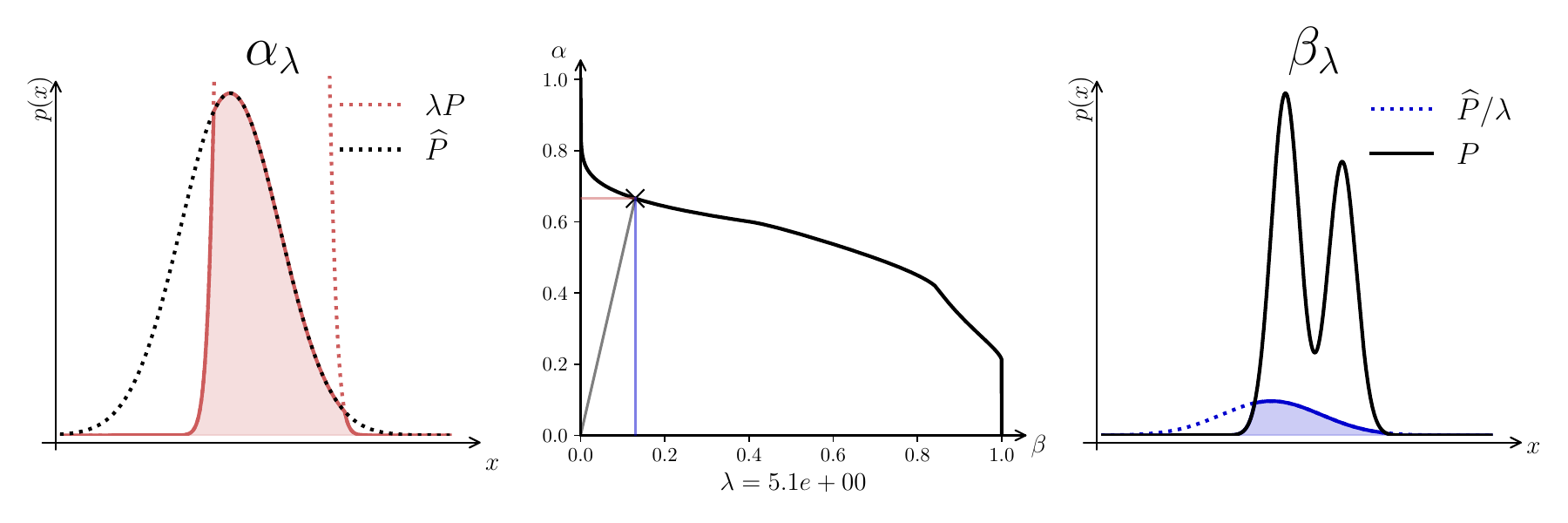}        \label{app:fig:goodrecallhighl}}
        \caption{PR-Curves for the model in Figure~\ref{app:fig:goodrecall}}
    \end{figure}
\end{minipage}
\end{figure}

\begin{figure}[H]
\begin{minipage}[c]{0.3\textwidth}
    \begin{figure}[H]
        \centering
\includegraphics[width=\textwidth]{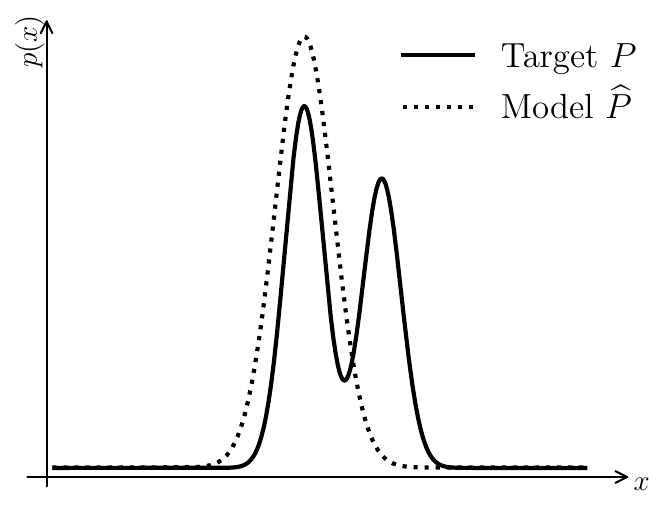}
        \caption{A target distribution $P$ and the approximated distribution $\whP$. In this setup, the model is expected to have a poor recall since it covers almost only half the weight of $P$ but a decent precision since the weight covers the contours of $P$ well. In Figures~\ref{app:fig:goodprecisionlowl}~and~\ref{app:fig:goodprecisionhighl}, we represent the PR-Cruve and how it is computed.}
        \label{app:fig:goodprecision}
    \end{figure}
\end{minipage}
\hfill
\begin{minipage}[c]{0.65\textwidth}
    \begin{figure}[H]
        \subfloat[PR-Curve for the model in Figure~\ref{app:fig:goodprecision}, explained for a low $\lambda$. The area in red is $\alpha_\lambda$ and the area in blue is $\beta_\lambda$.]{\includegraphics[width=\textwidth]{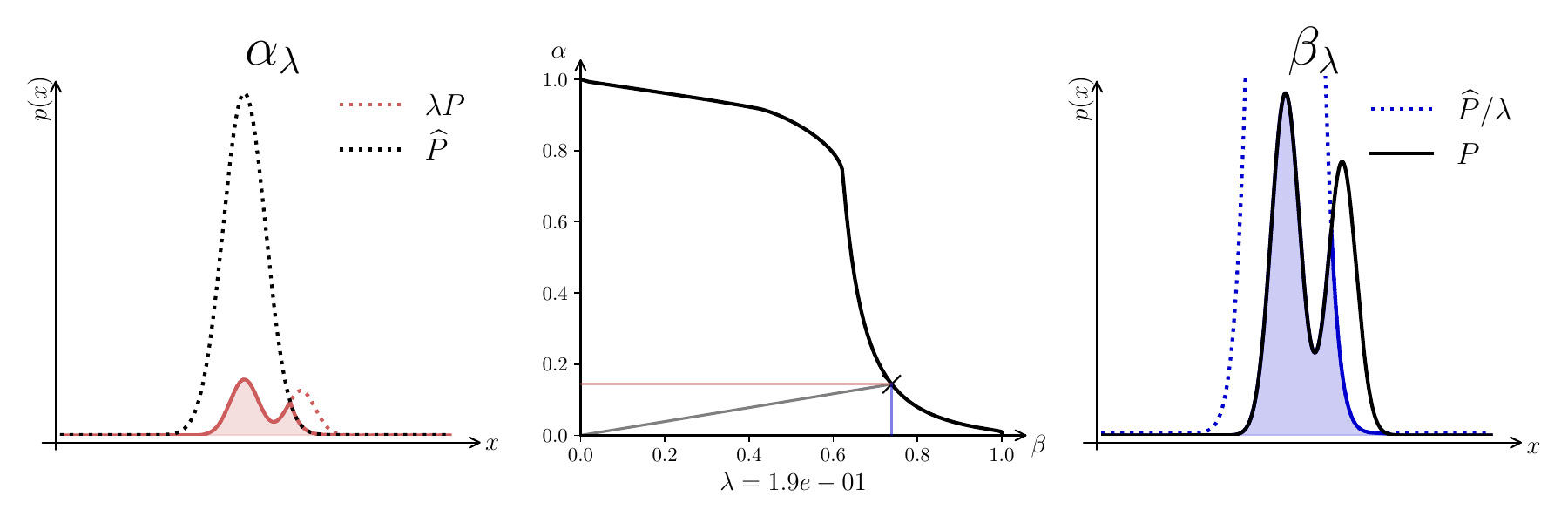}\label{app:fig:goodprecisionlowl}}
        
        \subfloat[PR-Curve for the model in Figure~\ref{app:fig:goodprecision}, explained for a high $\lambda$. The area in red in $\alpha_\lambda$ and the area in blue is $\beta_\lambda$.]{\includegraphics[width=\textwidth]{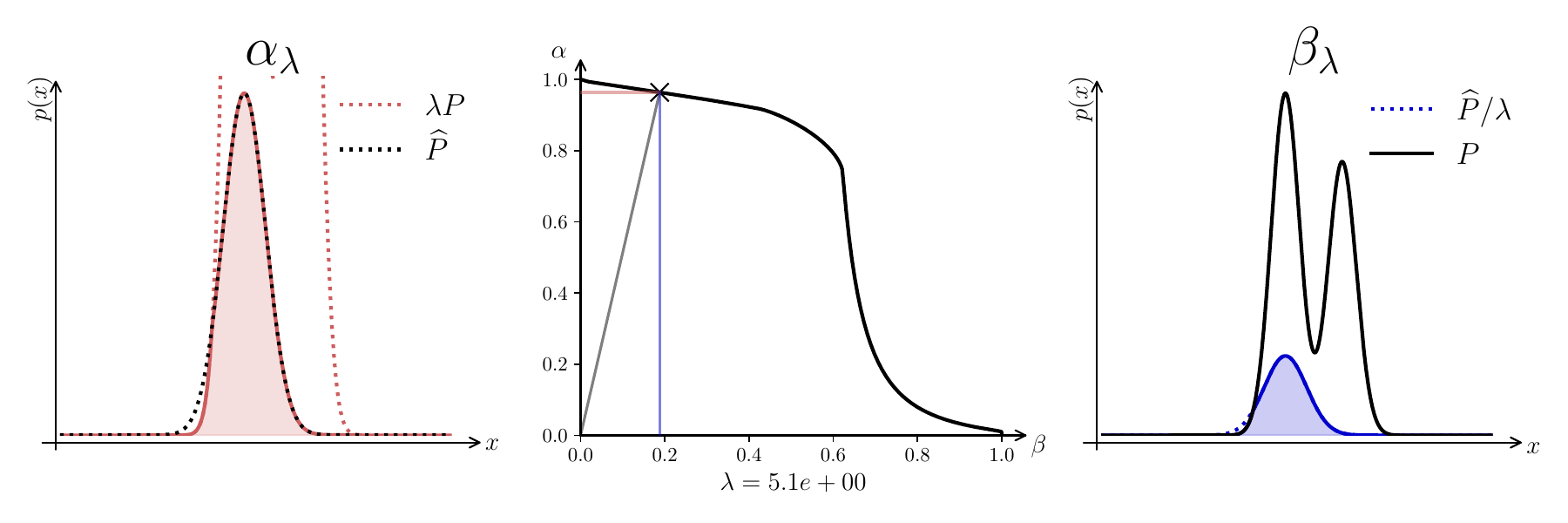}        \label{app:fig:goodprecisionhighl}}
        \caption{PR-Curves for the model in Figure~\ref{app:fig:goodprecision}}
    \end{figure}
\end{minipage}
\end{figure}

\subsection{Precision and Recall in practice}
To perfectly compute the set $\PRd(P, \whP)$, one needs the ratio $p(\vx)/\whp(\vx)$ for all $\vx\in \Xset$. In practice, a variety of heuristics are employed.
\cite{sajjadi_assessing_2018} use $k$-NN based algorithm in the Inception latent space to estimated the densities. \cite{simon_revisiting_2019} use an ensemble of classifiers in Inception's latent space to estimate the likelihood ratio. \cite{verine_precision-recall_2023}  use a neural network based discriminator, simlarly to $f$-GANs, to estimate the likelihood ratio. With these methods, we can compute the PR-Curve for high dimensionnal dataset such as MNIST: see Figure~\ref{app:fig:PRMNIST}.

 \begin{figure}[H]
 \centering
    \subfloat[Model 1: High Recall\\FID: $17.06$, IS: $2.69$]{\hspace{0.03\textwidth}\includegraphics[width=0.15\textwidth]{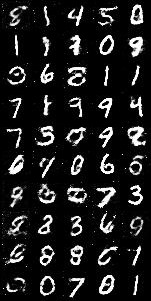}\label{fig:model1}\hspace{0.07\textwidth}}
    \subfloat[Model 2: High Precision\\FID: $8.80$, IS:$2.57$]{\hspace{0.03\textwidth}\includegraphics[width=0.15\textwidth]{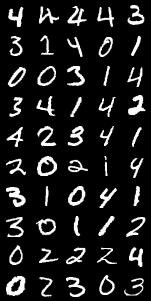}\hspace{0.07\textwidth}}
    \subfloat[PR-Curves for Model 1 and 2.]{\hspace{0.07\textwidth}\includegraphics[width=0.25\textwidth]{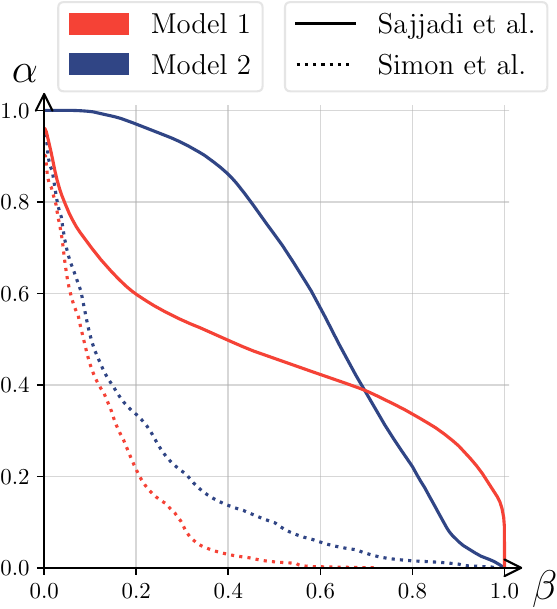}\label{fig:prcruves12} \hspace{0.07\textwidth}}
    \caption{Two different models are displayed with very different performances. Model 1 have a great diversity and display all different digits, but contours, backgrounds and shapes are sometimes incoherent. Model 2 is generating coherent samples from only half the classes. Traditional metrics - FID ($\downarrow$) and IS ($\uparrow$) - are given for comparison. }
    \label{app:fig:PRMNIST}
\end{figure}
\section{Proof and Supplementary for Section~\ref{sec:optirej}}
\subsection{Proof for Theorem~\ref{thm:optrej}}\label{app:sec:proofoptirej}
The goal is to find an acceptance function $a(\vx)$ that first minimizes the \fdiv between the target distribution $P$ and the distribution after the rejection process $\wtP_a$. A budget is added to the problem in order to avoid low acceptance rate. We set the budget to be $K$, the average number of samples to draw before accepting one. With a budget of $K$, the average acceptance rate is  $1/K$. In analogy with the unlimited budget rejection process, the average number of samples to draw in order to keep one is $M=\max_{\setX} p(\vx)/\whp(\vx)$. The function $a$ is the solution of the problem:
\begin{align}
	\begin{split}
		\min_a & \quad\Df(P\Vert\wtP_a ) \\
		\mbox{s.t.} &\quad \begin{cases} \mathbb{P}\left(\mbox{acceptance}\right)\geq1/K \\
		\forall \vx, \, 0\leq a(\vx) \leq 1 \end{cases}
	\end{split}
\end{align}
First, we can consider $\Df(\wtP_a\Vert P )$ instead of $\Df(P\Vert\wtP_a )$ without loss of generality: This is because  $\Df(P\Vert\wtP_a ) = \mathcal{D}_{f'}(\wtP_a\Vert P )$ for $f': x\mapsto xf(1/x)$. Further, the solution to the optimal $a(\vx)$ turns out to be independent of $f$.

Moreover, we can assume that the budget is always lower that the unlimited budget. In other terms, instead of forcing the acceptance rate to be greater to $1/K$ we can force is to be exactly equal to $1/K$. Then, the probability of acceptance being $\mathbb{P}\left(\mbox{acceptance}\right) = \E_{\whP}\left[a(\vx)\right]$,  we can write an equivalent problem as: 

\begin{align}
	\begin{split}
		\min_a & \quad\Df(\wtP_a\Vert P ) \\
		\mbox{s.t.} &\quad \begin{cases} \E_{\whP}\left[a(\vx)\right] = 1/K \\
		\forall \vx, \, 0\leq a(\vx) \leq 1 \end{cases}
	\end{split}
\end{align}

Using the definition of the densities in the rejection sampling context, $\wtp_a(\vx) = K\whp(\vx)a(\vx)$, the problem is equivalent to:
\begin{align}
	\begin{split}
		\min_a & \quad \E_{P}\left[f\left(\frac{K\whp(\vx)a(\vx)}{p(\vx)}\right)\right] \\
		\mbox{s.t.} &\quad \begin{cases} \E_{\whP}\left[a(\vx)\right]= 1/K \\
		\forall \vx, \, 0\leq a(\vx) \leq 1 \end{cases}
	\end{split}
\end{align}

Switching to the discrete case, the problem becomes :
\begin{align}
\begin{split}
\label{eq:app:discreteprob}
	\min_{\va\in \reals^N} & \quad \sum_i^N p_i f\left(a_i \frac{\wh p_iK}{p_i}\right)  \\
    \mbox{s.t.} &\quad  \begin{cases} \sum_i^N \wh p_i a_i = 1/K \\
     \forall i, \, 0\leq a_i \leq 1   \end{cases}
\end{split}
\end{align}

The Lagrangian function associated with the problem  \ref{eq:app:discreteprob} is :
\begin{align}
    \mathcal{L}(\va, \mu ,\vlambda_1, \vlambda_2) = \sum_i^N p_if\left(a_i \frac{\wh p_iK}{p_i }\right) + \mu \left[ \va^T \wh \vp- 1/K\right] + (\va-\mathds{1})^T\vlambda_1 - \va^T\vlambda_2
\end{align}
All constraints are affine and the objective function is a convex function, therefore the optimal vector $\va\s$ satisfies the KKT conditions: 
\begin{equation}
\label{eq:KKTl1l2}
 \begin{cases}
\nabla_{a_i} \mathcal{L}(\va\s, \mu\s \vlambda_1\s, \vlambda_2\s) = K\wh p_i\nabla f\left(a_i\s \frac{\wh p_iK}{p_i }\right) + \mu\s\wh p_i + (\lambda_{1i}\s-\lambda_{2i}\s) = 0,  \quad \forall i \\
\sum_i a\s_i \wh p_i=  1/K \\
\lambda_{1i}\s(a_i\s-1)=0, \quad \forall i \\
\lambda_{2i}\s a_i\s=0, \quad \forall i \\
\lambda_{1i}\s, \lambda_{2i}\s\geq 0, \forall i
\end{cases}
\end{equation}
Using the 1st condition: 
\begin{align}
    a_i\s= \frac{p_i }{\wh p_i K } \left[\nabla f\right]\inv\left(\frac{\lambda_{2i}\s-\lambda_{1i}\s}{K\wh p_i} - \mu/K\right)
\end{align}

Since $\left[ \nabla f\right]\inv = \nabla f\s$:
\begin{align}
\label{eq:aifs}
    a_i\s= \frac{p_i }{\wh p_iK } \left[\nabla f\s
    \right]\left(\frac{\lambda_{2i}\s-\lambda_{1i}\s}{\wh p_iK} - \mu/K\right)
\end{align}
If the Pearson $\chi^2$ is put aside, all the usual $f\s$ are strictly increasing functions. Therefore, according to Eq~\ref{eq:aifs}, all $a_i>0$. Thus all $\lambda_{2i}\s=0$.
The KKT conditions~\ref{eq:KKTl1l2} become :

\begin{equation}
 \begin{cases}
K\wh p_i \nabla f\left(a_i\s \frac{\wh p_iK}{p_i }\right) + \mu\s \wh p_i  +\lambda_{1i}\s = 0,  \quad \forall i \\
\sum_i a\s_i \wh p_i= 1/ K \\
\lambda_{1i}\s (a_i\s-1)=0, \forall i \\
\lambda_{1i}\s\geq 0, \forall i
\end{cases}
\end{equation}
And thus :
\begin{align}
\label{eq:aifs2}
    a_i\s= \frac{p_i }{\wh p_i K } \left[\nabla f\s
    \right]\left(-
    \frac{\lambda_{1i}\s}{\wh p_iK} - \mu/K\right)
\end{align}
To get the full formula for $a_i\s$, we need to compute the $\lambda_{1i}$s. For this purpose, let us use strong duality to reformulate our problem:
\begin{align}
    \min_{\va} \max_{\vlambda \geq \vzero, \vmu} &\sum_i^N p_if\left(a_i  \frac{\wh p_iK}{p_i }\right) + \mu \left[ \va^T \wh \vp- 1/K\right] + (\va-\mathds{1})^T\vlambda_1 \\
    &= \max_{\vlambda \geq \vzero, \vmu}  \min_{\va}\sum_i^N p_if\left(a_i \frac{\wh p_iK}{p_i }\right) + \mu \left[ \va^T \wh \vp- 1/K\right] + (\va-\mathds{1})^T\vlambda_1
\end{align}
Then, we can use the Fenchel Conjugate:
\begin{align}
	\begin{split}
		\min_{\va}\sum_{i}^{N}p\s_{i}f\left(a_i\frac{\wh p_{i}K}{p\s_{i}}\right)+\mu\left[\va^{T}\wh\vp-1/K\right]+(\va-\mathds{1})^{T}\vlambda_{1} & = \begin{multlined}[t]\min_{\va}\sum_{i}^{N}p\s_{i}\left[f\left(a_i\frac{\wh p_{i}K}{p\s_{i}}\right)-a_{i} \left(\frac{-\mu\wh p_{i}-\lambda_{1i}}{p_i}\right)\right]\\
 -\mu /K-\mathds{1}^{T}\vlambda_{1} \end{multlined}\\
		 & =\begin{multlined}[t]-\sup_{\va}\left\{ \sum_{i}^{N}p\s_{i}\left[a_{i}\left(\frac{-\mu\wh p_{i}-\lambda_{1i}}{p_i}\right)-f\left(a_i\frac{\wh p_{i}K}{p\s_{i}}\right)\right]\right\} \\ -\mu /K-\mathds{1}^{T}\vlambda_{1}\end{multlined} \\
	 & =-\sum_{i}^{N}\left[p\s_{i}f\s\left(-\frac{p\s_{i}}{\wh p_{i}K}\frac{\mu\wh p_{i}+\lambda_{1i}}{p_i}\right)\right]-\mu /K-\mathds{1}^{T}\vlambda_{1}\\
 & =-\sum_{i}^{N}\left[p\s_{i}f\s\left(-\mu/K-\frac{\lambda_{1i}}{\wh p_{i}K}\right)\right]-\mu/ K-\mathds{1}^{T}\vlambda_{1}
    \end{split}
\end{align}

Define $u_{i}=\frac{\lambda_{i1}}{\hat{p}_{i}}$, assuming $\hat{p}_{i}>0$
everywhere. Note that the constraints $\lambda_{i1}\ge0$ and $u_{i}\ge0$
are equivalent. The above equation becomes

\begin{align}
\sup_{\lambda_1\ge0}\mathcal{L}\left(\va\s,\mu\s, \vlambda_{1},\vlambda_{2}\s\right) & =\sup_{\vu\ge0}-\sum_{i}^{N}p\s_{i}f\s\left(-\left(\mu\s+u_{i}\right)/K\right)-\sum_{i}^{N}\hat{p}_{i}u_{i}-\mu\s/K
\end{align}

Let us make another change of variable to make a conjugate form appear.
Define $v_{i}=-\left(\mu\s+v_{i}\right)$. So $u_{i}=-\mu\s-v_{i}$
and the constraint $u_{i}\ge0$ becomes $v_{i}\le-\mu\s$.
Also, define $g(t)=f(Kt)$. Then $g\s  (t)=f\s(\frac{t}{K})$.
Above equation becomes 

\begin{align}
	\sup_{\vlambda_1\ge0}\mathcal{L}\left(\va\s,\mu\s, \vlambda_{1},\vlambda_{2}\s\right) & =\sup_{\vv\le-\mu\s}\sum_{i}^{N}\hat{p}_{i}v_{i}-\sum_{i}^{N}p_i g\s  \left(v_{i}\right)-\mu\s\left(K-1\right)
\end{align}

Recall that $\arg\sup_{t}\left\langle a,t\right\rangle -f(t)=\nabla f\s(a)$
and $\arg\sup_{t}\left\langle a,t\right\rangle -f\s(t)=\nabla f(a)$.
Thus, given $\mu\s$ we can compute the optimal values of $v_{i}$
one by one as follows:

\begin{align*}
v_{i}\s & =\arg\sup_{v_{i}\le-\mu\s}\hat{p}_{i}v_{i}-p_i g\s  \left(v_{i}\right)\\
 & =\arg\sup_{v_{i}\le-\mu\s}\frac{\hat{p}_{i}}{p_i}v_{i}-g\s  \left(v_{i}\right)\\
 & =\min\left(-\mu\s,\nabla g\left(\frac{\hat{p}_{i}}{p_i}\right)\right)
\end{align*}

So $u_{i}\s=\max\left(0,-\mu\s-\nabla g\left(\frac{\hat{p}_{i}}{p_i}\right)\right)$. This gives us the optimal values of $\lambda_{i1}\s$.
Note that $\nabla g(t)=K\nabla f\left(Kt\right)$. Replacing $\frac{\lambda_{1i}\s}{\hat{p}_{i}}$ by $u_{i}\s$
in the formula of $a_{i}\s$ gives us:

\begin{align*}
a_{i}\s & =\frac{p_i }{\hat{p}_{i}K}\nabla f\s\left(-\mu\s/K-\max\left(0,-\mu\s-\nabla g\left(\frac{\hat{p}_{i}}{p_i}\right)/K\right)\right)\\
 & =\frac{p_i }{\hat{p}_{i}K}\nabla f\s\left(-\mu\s/K+\min\left(0,\mu\s+\nabla g\left(\frac{\hat{p}_{i}}{p_i}\right)/K\right)\right)\\
 & =\frac{p_i }{\hat{p}_{i}K}\nabla f\s\left(\min\left(-\mu\s,\nabla g\left(\frac{\hat{p}_{i}}{p_i}\right)\right)/K\right)\\
 & =\frac{p_i }{\hat{p}_{i}K}\nabla f\s\left(\min\left(-\mu\s/K,\nabla f\left(\frac{\hat{p}_{i}K}{p_i}\right)\right)\right)
\end{align*}

Note that $\nabla f\s$ is strictly increasing, thus:

\begin{align*}
	a_{i}\s & =\frac{p_i }{\hat{p}_{i}K}\min\left(\nabla f\s\left(-\frac{\mu\s}{K}\right),\frac{\hat{p}_{i}K}{p_i}\right)\\
 & =\min\left(\frac{p_i }{\hat{p}_{i}K}\nabla f\s\left(-K\mu\s\right),1\right)
\end{align*}

Note that $\nabla f\s\left(-\mu\s/K\right)$ is a constant.
So the optimal acceptance function under budget looks like $a(\vx)=\min\left(1,c\frac{p(\vx)}{\hat{p}(\vx)}\right)$
for some constant $c$ defined by K only as:
\begin{align}
	\int_{\setX} \min\left(\whp(\vx),cp(\vx)\right)\dx = 1/K.
\end{align}
To facilitate the understanding of $c$, we can set this constant to be equal to $c/M$ instead. Thus,
\begin{align}
a(\vx) = \min\left(\frac{p(\vx)}{\whp(\vx)}\frac{c}{M}, 1\right)
\end{align}
With that notation, $c\geq1$ and if the optimal unlimited acceptance function is obtained with $c=1$:  
\begin{align}
	a(\vx) = \min\left(\frac{p(\vx)}{\whp(\vx)}\frac{1}{M}, 1\right) = \frac{p(\vx)}{\whp(\vx)M}
\end{align}

\subsection{Algorithm to compute $c_K$} \label{app:subsec:algoc}
In Section~\ref{sec:optirej},  we show that the optimal acceptance function is
\begin{align}
a(\vx, c_K) =\min\left(\frac{p(\vx)}{\whp(\vx)}\frac{c_K}{M}, \, 1\right).
\end{align}
The constant $c_K$ is determined exclusively by the budget $K$. In practice, we can draw a set of samples from $\whP$ and adjust $c_K$ to obtain the correct budget. We use a dichotomy algorithm detailed in Algorithm~\ref{alg:searchc}.
\begin{algorithm}[H]
\caption{Dichotomy to compute $c_K$.}
\label{alg:searchc}
\textbf{Input}: N generated samples $\vx^{\mathrm{fake}}_1, \dots, \vx^{\mathrm{fake}}_N\sim \whP$ \\
\textbf{Parameter}: Budget $K$, Threshold $\epsilon$\\
\textbf{Output}: Constant $c_K$
\begin{algorithmic}[1] 
\State Let $c_{\mathrm{min}}=1e^{-10}$ and $c_{\mathrm{max}}=1e^{10}$.
\State $c_K = (c_{\mathrm{max}}+c_{\mathrm{min}})/2$
\State Define the loss $\cal L(c_K) = \sum_{i=1}^Na\left(\vx_i^{\mathrm{fake}}, c_K\right)-\frac{1}{K}$
\While{ $\left\vert\cal L(c_K)\right\vert\geq \epsilon$}
\If{$\cal L(c_K)>\epsilon$}
\State $c_{\mathrm{max}}=c_K$
\ElsIf{$\cal L(c_K)<-\epsilon$}
\State $c_{\mathrm{min}}=c_K$
\EndIf
\State Update: $c_K = (c_{\mathrm{max}}+c_{\mathrm{min}})/2$
\State Update: $\cal L(c_K)$

\EndWhile

\end{algorithmic}
\end{algorithm}

\subsection{Proof for Theorem~\ref{thm:improvalpha}}\label{app:sec:improvalpha}
First, with $a(\vx) = \min\left(1, \frac{c_k}{M}\frac{p(\vx)}{\whp(\vx)}\}\right)$, let us recall that 
\begin{align}
    \wtp_a(\vx) &= K \whp(\vx)a(\vx) \\
    & = \min \left(K\whp(\vx), \frac{Kc_K}{M}p(\vx)\right).
\end{align}
Thus:
\begin{align}
    \alpha_\lambda(P\Vert \wtP_a) &= \int_{\Xset}\min\left(\lambda p(\vx), \,\wtp(\vx)\right)\dx \\
&=  \int_{\Xset}\min\left(\lambda p(\vx), \, K\whp(\vx),\, \frac{Kc_K}{M}p(\vx) \right)\dx.
\end{align}
Naturally, the precision can be evaluated for $\lambda$ lower or greater than $Kc_K/M$. 
For $\lambda\leq Kc_K/M$:
\begin{align}
    \alpha_\lambda(P\Vert \wtP_a) &=  \int_{\Xset}\min\left( K\whp(\vx),\, \frac{Kc_K}{M}p(\vx) \right)\dx \\
    &=   K\int_{\Xset}\min\left(\frac{c_K}{M}p(\vx),\, \whp(\vx) \right)\dx \\
    &=  K \E_{\whP}\left[\min\left(\frac{c_K}{M}\frac{p(\vx)}{p(\vx)},\, 1 \right) \right]  \\
    &= K \frac{1}{K} \quad \mbox{by definition of} c_K, \\
    &=  K\alpha_{c_K/M}(P\Vert \whP).
\end{align}
Thus, under a given threshold $Kc_K/M$, the precision is constant and equal to $K\alpha_{c_K/M}(P\Vert \whP)$. Moreover, we can give a lower bound on this constant value in terms of $K$. As a matter of fact, $\alpha_\lambda$ is an increasing function of $\lambda$, therefore:
\begin{align}
    \alpha_{c_K/M}(P\Vert \whP) &\geq  \alpha_{1/M}(P\Vert \whP) \\
    &\geq  \int_{\Xset}\min\left(\frac{1}{M}p(\vx), \,\wtp(\vx)\right)\dx.
\end{align}
Finally, by the definition of $M$, for every $\vx\in\Xset$, $\frac{1}{M}p(\vx)\leq \whp(\vx)$. Consequently, 
\begin{align}
  \alpha_\lambda(P\Vert \wtP_a) =   K\alpha_{c_K/M}(P\Vert \whP) \geq \frac{K}{M}.
\end{align}
For $\lambda\leq Kc_K/M$:
\begin{align}
    \alpha_\lambda(P\Vert \wtP_a) &=  \int_{\Xset}\min\left( \lambda p(\vx),\,K\whp(\vx) \right)\dx \\
    &=   K\int_{\Xset}\min\left(\frac{\lambda}{K}p(\vx),\, \whp(\vx) \right)\dx \\
    &=  K\alpha_{\lambda/K}(P\Vert \whP). 
\end{align}
And, since $\lambda/K\geq \lambda/M$:
\begin{align}
    \alpha_\lambda(P\Vert \wtP_a)=  K\alpha_{\lambda/K}(P\Vert \whP)\geq  K\alpha_{\lambda/M}(P\Vert \whP).
\end{align}
Finally, with $\alpha_\lambda = \lambda \beta_\lambda $, 
\begin{align}
    \beta_\lambda(P\Vert \wtP_a) = \frac{K}{\lambda}\alpha_{\lambda/K}(P\Vert \whP) = \frac{K}{(\lambda)}\frac{\lambda}{K}\beta_{\lambda/K}(P\Vert \whP) = \beta_{\lambda/K}(P\Vert \whP), 
\end{align}
And, since $\lambda/K\leq M/c_K$, we have:
\begin{align}
    \beta_\lambda(P\Vert \wtP_a)\geq \beta_{c_K/M}(P\Vert \whP), 
\end{align}
Therefore we have two regimes:
\begin{itemize}
    \item For $\lambda\geq \frac{Kc_K}{M}$:
    \begin{align*}
        \alpha_\lambda \left(P\Vert \wtP_{\aobs}\right) = 1 \et
        \beta_\lambda \left(P\Vert \wtP_{\aobs}\right) = 1/\lambda
    \end{align*}
    \item For $\lambda\leq \frac{Kc_K}{M}$:
    \begin{align*}
    \begin{cases}
    \alpha_\lambda(P\Vert \wtP_{\aobs})= K \alpha_{\lambda/K}(P\Vert \whP) \\
        \beta_\lambda(P\Vert \wtP_{\aobs})=  \beta_{\lambda/K}(P\Vert \whP)
    \end{cases}
    \end{align*}
    This can be seen as a vertical scaling of the PR-Curve. For a given point $(\alpha, \beta)$ in $\PRd(P\Vert \whP)$, then the point with the same $\beta$ in  $\PRd(P\Vert \wtP)$ has a precision $K\alpha$, up to a certain saturating level ($\alpha<1$). 
\end{itemize}\newpage

\subsection{Information Divergence Frontier Improvement}\label{app:sec:djolonga}

In \cite{djolonga_precision-recall_2020}, the authors  define another precision-recall curve, named the Information Divergence Frontiers:

\[
\mathcal{F}_{\beta}^{\cap}\left(P,Q\right)=\left\{ (\pi,\rho)\in\mathcal{R}_{\beta}^{\cap}\left(P,Q\right):\not\exists(\pi',\rho')\in\mathcal{R}^{\cap}\left(P,Q\right)\,s.t.\,\text{\ensuremath{\pi'<\pi,\rho'<\rho}}\right\} 
\]

Where $\mathcal{R}_{\beta}^{\cap}\left(P,Q\right)=\left\{ \left(\mathcal{D}_{\beta}(R,Q),\mathcal{D}_{\beta}(R,P)\right):R\in\mathcal{P}(\mathcal{X})\right\} $
and where $\mathcal{D}_{\beta}$ is the Renyi divergence parametrized
by $\beta$.

As an immediate corrolary of the previous theorem and of proposition
6 of \cite{djolonga_precision-recall_2020}, we can write the following:

\begin{corollary}

Under the same setting as theorem \ref{thm:improvalpha}, for any $(\pi,\rho)\in\mathcal{F}_{\infty}^{\cap}\left(P,\widehat{P}\right)$
we have $(\pi',\rho)\in\mathcal{F}_{\infty}^{\cap}\left(P,\widetilde{P}_{a_{O}}\right)$
with $\pi'=\max\left(0,\pi-\log K\right)$. 

\end{corollary}

\section{Bounds}\label{app:sec:bounds}
\begin{theorem}
Let $M=\sup_{x\in\mathcal{X}}\frac{p(x)}{\whp(\vx)}$. For any $f$-divergence,
we have 
\[
\Df(P\Vert\wtP_a)\le\Df(P\Vert\whP)-\min\left(1,\frac{K-1}{M}\right)\Df(P\Vert\whP)
\]

and for Kullback-Leibler we have for $\beta=\frac{\log K}{\log M}$

\[
\KL(P\Vert\wtP)\le(1-\beta)\left(\KL(P\Vert\whP)-\D^\mathrm{R}_{\beta}(P\Vert\whP)\right)
\]
where $\D^\mathrm{R}_{\beta}$ is the Renyi divergence with parameter $\beta$ 
\end{theorem}
\begin{proof}
For both bounding the $f$-divergence and the KL divergence, the strategy
will be the same. We want to show that 
\[
\Df(P\Vert\whP)-\Df(P\Vert\wtP)\ge\text{some lower bound}
\]

Note that for any density $p_{\alpha}$ such that such that $p_{\alpha}\le K\whp$,
we have $\Df(P\Vert\wtP)\le\Df(P,P_{\alpha})$ so

\[
\Df(P\Vert\whP)-\Df(P\Vert\wtP)\ge\Df(P\Vert\whP)-\Df(P,P_{\alpha})
\]

So once we have a suitable $p_{\alpha}$, we need to show the lower
bound holds:

\[
\Df(P\Vert\whP)-\Df(P,P_{\alpha})\ge\text{some lower bound}
\]

For bounding general $f$-divergences, we will choose $p_{\alpha}=\whp+\alpha\left(p-\whp\right)$
with $\alpha=\min\left(1,(K-1)\inf_{x\in\mathcal{X}}\frac{\whp(\vx)}{p(\vx)}\right)$

Let us first show that $p_{\alpha}\le K\whp$. 

\begin{align*}
p_{\alpha} & \le\whp+(K-1)\inf_{x}\frac{\whp(\vx)}{p(\vx)}\left(p-\whp\right)
\end{align*}

Note that for any $\vx'\in\mathcal{X}$, 
\begin{align*}
\inf_{x}\frac{\whp(\vx)}{p(\vx)}\left(p(\vx')-\whp(\vx')\right) & \le\whp(\vx')
\end{align*}

So 
\begin{align*}
p_{\alpha}(\vx) & \le\whp+(K-1)\whp\le K\whp
\end{align*}

Next, let us show the lower bound. Recall that $\Df\left(p,\cdot\right)$
is convex in its second argument. Thus, convexity implies:

$\Df(P\Vert P_{\alpha})\le\left(1-\alpha\right)\Df(P\Vert\whP)+\alpha\Df(P\Vert P)\le\left(1-\alpha\right)\Df(P\Vert\whP)$

Now to apply the same type of idea to bound the KL, let us define
$p_{\beta}(\vx)=\frac{1}{Z}\whp(\vx)^{1-\beta}p(\vx)^{\beta}$, where $Z=\int\whp(\vx)^{1-\beta}p(\vx)^{\beta}d\mu(\vx)=e^{(\beta-1)\D^\mathrm{R}_{\beta}(P\Vert\whP)}$
where $\D^\mathrm{R}_{\beta}(P\Vert\whP)=\frac{1}{\beta-1}\log\int p^{\beta}\whp^{1-\beta}d\mu$
is the Renyi divergence of parameter $\beta$ and $\mu$ is the reference
measure.

First, let us choose $\beta'=\frac{\log K-(1-\beta')R_{\beta'}(P\Vert\whP)}{\log M}$
and let us show as before that $p_{\beta'}\le K\whp$. More precisely,
let us show that $\log\frac{p_{\beta'}(\vx)}{K\whp(\vx)}\le0$

For any $x$, we have

\begin{align*}
\log\frac{p_{\beta'}(\vx)}{K\whp(\vx)} & =(1-\beta')\log\whp(\vx)+\beta'\log p(\vx)-\log Z-\log K\whp(\vx)\\
 & =\log\whp(\vx)+\beta'\log\frac{p(\vx)}{\whp(\vx)}-(\beta'-1)R_{\beta'}(P\Vert\whP)-\log K\whp(\vx)\\
 & =\beta'\log\frac{p(\vx)}{\whp(\vx)}-(\beta'-1)R_{\beta'}(P\Vert\whP)-\log K\\
 & \le\frac{\log K-(1-\beta')R_{\beta'}(P\Vert\whP)}{\log M}\log\frac{p(\vx)}{\whp(\vx)}-(\beta'-1)\D^\mathrm{R}_{\beta}(P\Vert\whP)-\log K\\
 & \le\log K-(1-\beta')R_{\beta'}(P\Vert\whP)-(\beta'-1)R_{\beta'}(P\Vert\whP)-\log K\\
 & \le0
\end{align*}

More generally it is easy to see that for all $\beta\in[0,\beta']$, we have $p_{\beta}\le K\whp$. For convenience, we will choose $\beta=\frac{\log K}{\log M}$. Clearly, $\beta\le\beta'$ so $p_{\beta}\le K\whp$.
Finally, let us compute $\KL(P\Vert P_{\beta})$

\begin{align*}
\KL(P\Vert P_{\beta}) & =\int p(\vx)\log\frac{p(\vx).Z}{\whp(\vx)^{1-\beta}p(\vx)^{\beta}}d\mu(\vx)\\
 & =\int p(\vx)\log\left(\frac{p(\vx)^{1-\beta}}{\whp(\vx)^{1-\beta}}.Z\right)d\mu(\vx)\\
 & =(1-\beta)\int p(\vx)\log\left(\frac{p(\vx)}{\whp(\vx)}\right)d\mu(\vx)+\log Z\\
 & =(1-\beta)\KL(P\Vert\whP)-(1-\beta)\D^\mathrm{R}_{\beta}(P\Vert\whP)\\
 & =(1-\beta)\left(\KL(P\Vert\whP)-\D^\mathrm{R}_{\beta}(P\Vert\whP)\right)
\end{align*}

Thus the result holds: $\KL(P\Vert\wtP)\le \KL(P\Vert P_{\beta})\le(1-\beta)\left(\KL(P\Vert\whP)-\D^\mathrm{R}_{\beta}(P\Vert\whP)\right)$
\end{proof}

\section{Additional Experiments}
In this section, we provide more details on the different experiments. First, in Section~\ref{app:sec:MNIST}, we explain how the loss landscape is produced. Then, in Section~\ref{app:subsec:xp2D}, we provide more details on how the budget affects the results on the 25 Gaussians experiments. Finally, in Section~\ref{app:subsec:complexity}, we compare the traditional GAN training procedure and our approach in terms of time complexity.




\subsection{Smoothing the lanscape parameters for MNIST}\label{app:sec:MNIST}
Similarly to \cite{li_visualizing_2018}, the goal is to observe a two dimensional projection of the parameters domain of a neural network, and compute the loss on this domain.

To do so, we train a simple GAN on MNIST. Both the generator and the discriminator are based on 3 linear layers with Leaky Relu. The models are trained using the tradition approach described in Algorithm~\ref{alg:naiveapproach}. Let us define $\theta_0$, the parameter vector of the generator $G_{\theta_0}$. We randomly draw two directions $\theta_1$ and $\theta_2$ in the parameter domain: defining an hyperspace of generators defined as $G_{\theta_0+x\theta_1+y\theta_2}$ with $(x, y)\in\reals^2$.  Then, given any parameters $\theta$, we train a new discriminator $T_2$ based on samples of $P$ and $\whP_{G_{\theta}}$ to determine the baseline loss landscape ($K=1$). For the OBRS loss landscape, we fine-tune the initial model $T$ in order to perform optimal budgeted rejection sampling. Finally, similar to the baseline, a new discriminator $T_2$ is trained to estimate the loss, but based on samples $P$ and $\wtP$. 

In Figure~\ref{app:fig:MNIST}, we plot the loss surface. In addition to Figure~\ref{fig:MNISTSmooth}, we represent a batch of samples drawn from $G_{\theta_0}$ (lower left) and from the $G_\theta$ given the worst loss (upper right). When OBRS is applied, we show in red the rejected samples and in green the accepted samples.  
\begin{figure}[H]
    \centering
    \includegraphics[width=0.9\textwidth]{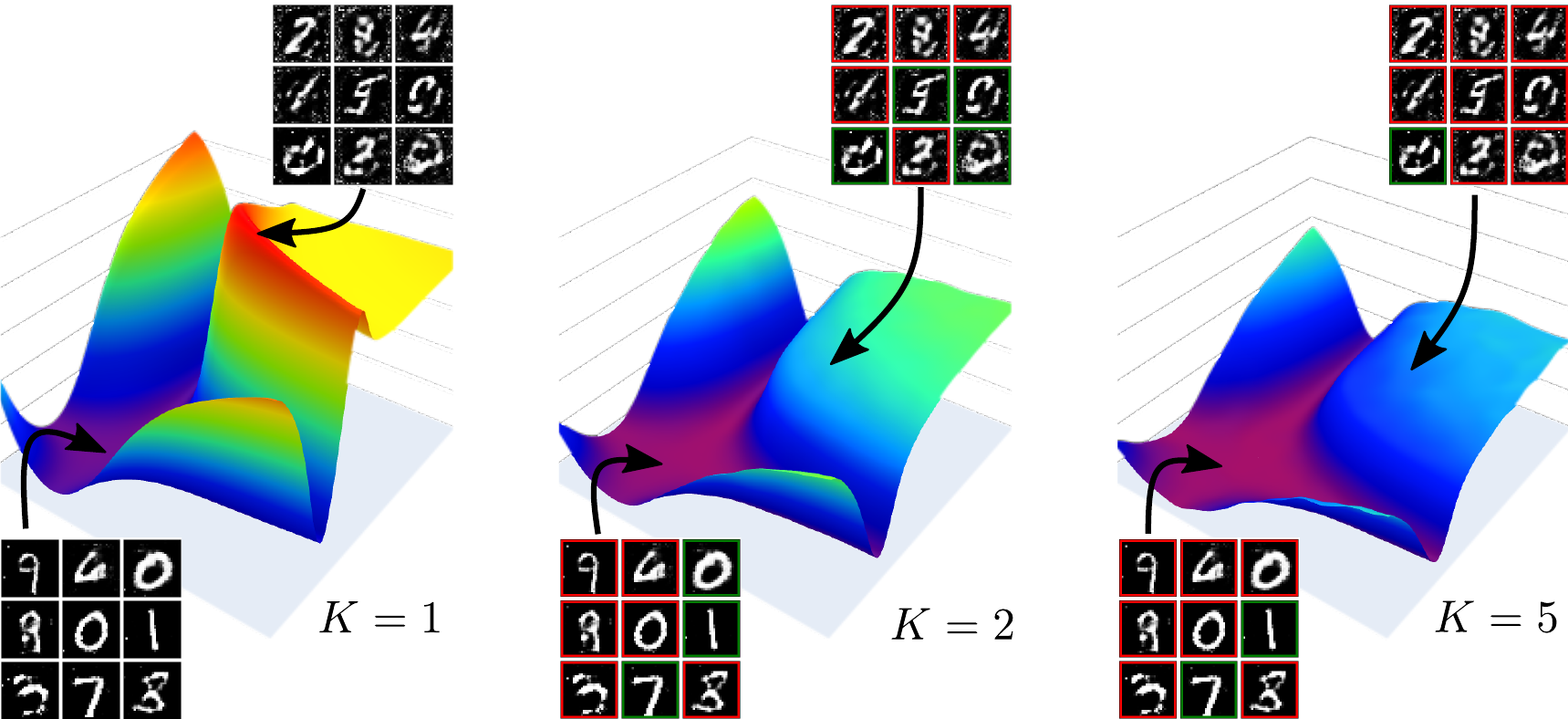}
    \caption{The Loss surface in the parameters domain of a DCGAN trained on MNIST randomly projected in 2D, observed for different rejection sampling budgets. }\label{app:fig:MNIST}
\end{figure}
 \subsection{Additional Results on the 2D 25-Gaussians Dataset}  \label{app:subsec:xp2D}
In this section, we provide more details on the GAN trained on the 25 Gaussians. The goal of this experiment is to compare OBRS with other rejection sampling methods such as DRS \citep{azadi_discriminator_2019} or MH-GAN \citep{turner_metropolis-hastings_2019}, but also with other sampling techniques that involve gradient descent, such as DOT \citep{tanaka_discriminator_2019} and DG$f$low \citep{ansari_refining_2021}. We train a simple GAN on 25 two-dimensional Gaussians and apply each method. We tune (when possible) the method to obtain around $6500$ inferences from the generator to generate $2500$ samples. To be more precise, both ORBS and DRS are easily tunable; however, the rejection rate of MH-GAN highly depends on the number of iterations of the algorithm. Therefore, we set the number of iterations to $2$ to obtain $40\%$ and then tune $\gamma$ and $K$ to achieve a similar acceptance rate.  We obtain the results plotted in Figure~\ref{app:fig:2Dviz}.
\begin{figure}[H]
     \centering
     \includegraphics[width=0.95\linewidth]{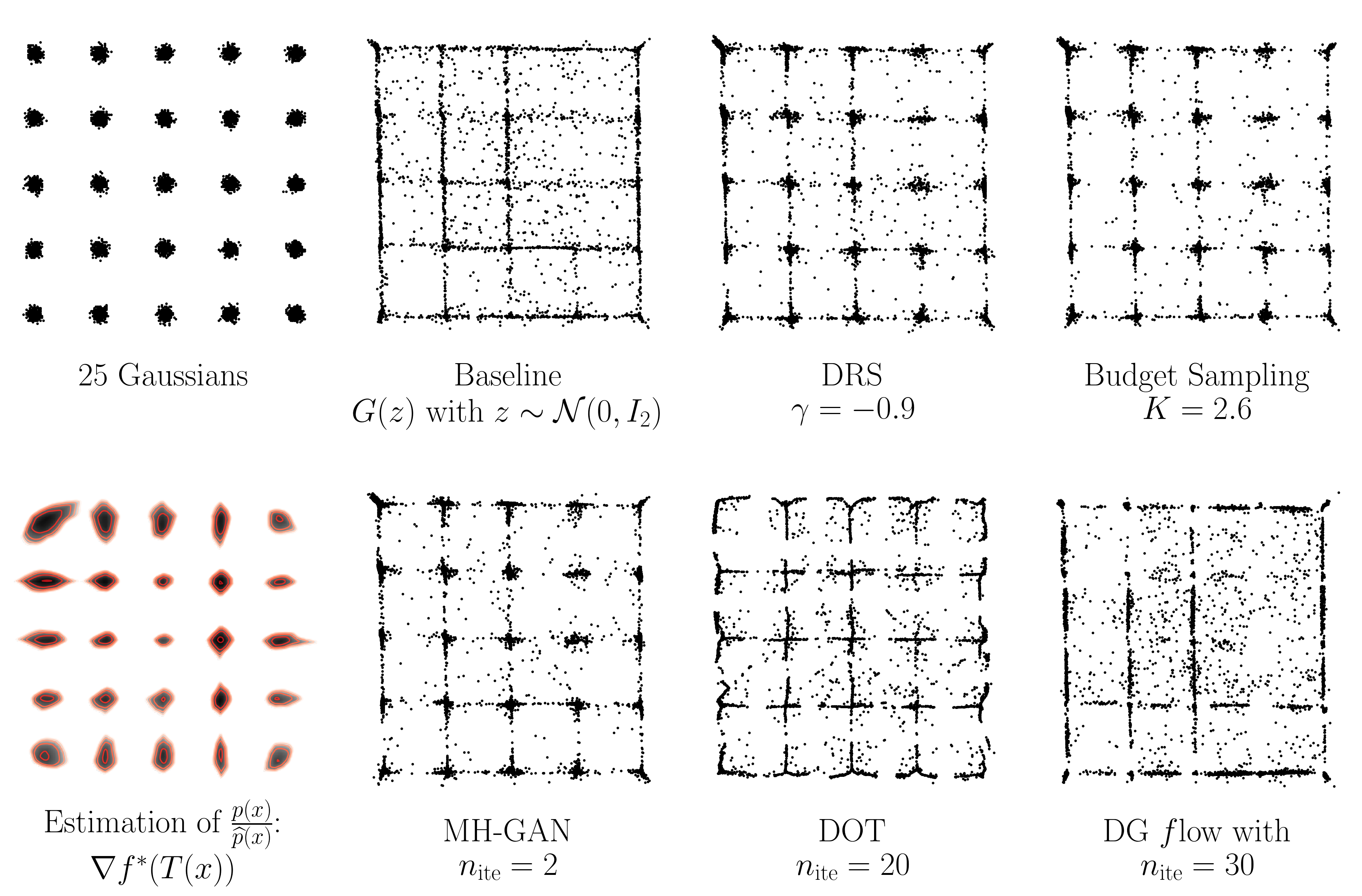}
     \caption{Visual Representation of the different sampling methods.}
     \label{app:fig:2Dviz}
 \end{figure}

In the previous experiment; we arbitrarily set up the acceptance rate (or the sampling time for the \emph{non} rejection sampling methods). We also compare the methods for different sampling time. Since for most methods, the recall was equivalent, we compare the precision denoted as \emph{high quality samples} in \cite{dumoulin_adversarially_2017}. We observe that for any given precision under $93\%$, the fastest method is the OBRS. However, OBRS,  like MH-GAN and DRS, appear to be capped.   
\begin{figure}[H]
    \centering
    \includegraphics[width=\linewidth]{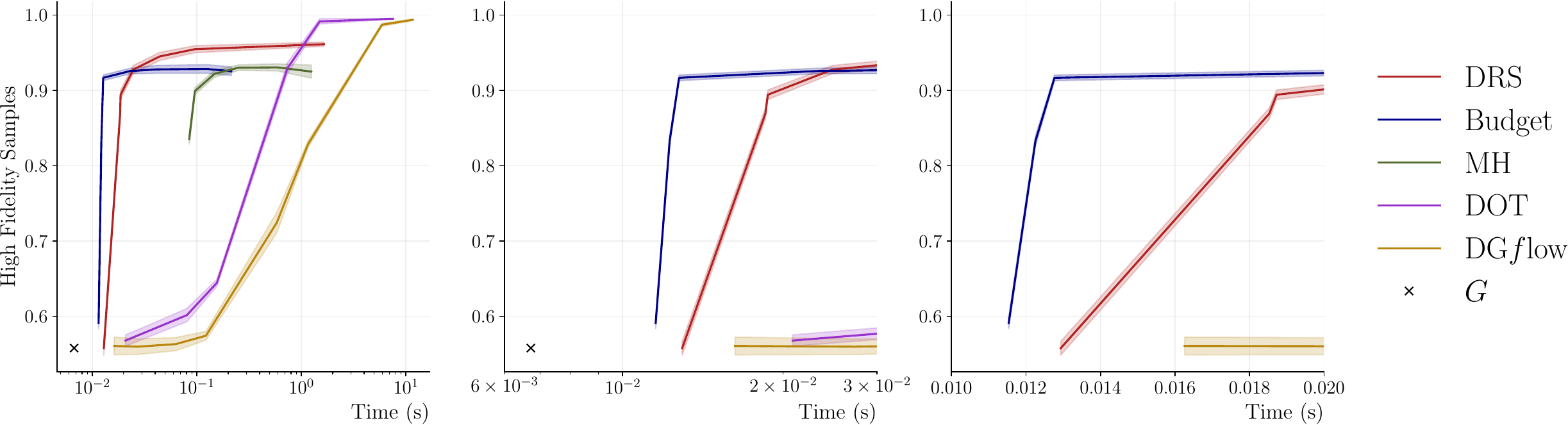}
    \caption{How the different methods behave with regard to time: to achieve similar results, DOT and DGflow need 100 more times. MH only 10 times more. And for similar time (similar budget), Budgeted Reject is better than DRS. (blue above red). MH, DRS and OBRS (Budget) are caped. They only use the discriminator to refine samples, while the DOT and DGflow sample data point from the latent space and refines the samples directly using Gradient ascent.}
    \label{fig:enter-label}
\end{figure}

As the distribution $\whP$ highly impacts how the rejection sampling methods behave, we also compare the OBRS and the DRS methods for different budgets and different $\whP$.  In Figure~\ref{app:fig:obrsvsdrs}, we observe that the precision of the OBRS is systematically better than the DRS. The distribution and the budget set in the experiment illustrated in Figures~\ref{app:fig:2Dviz} are set compare the methods for similar acceptance rates. 
\begin{figure}[H]
    \centering
    \includegraphics[width=0.8\linewidth]{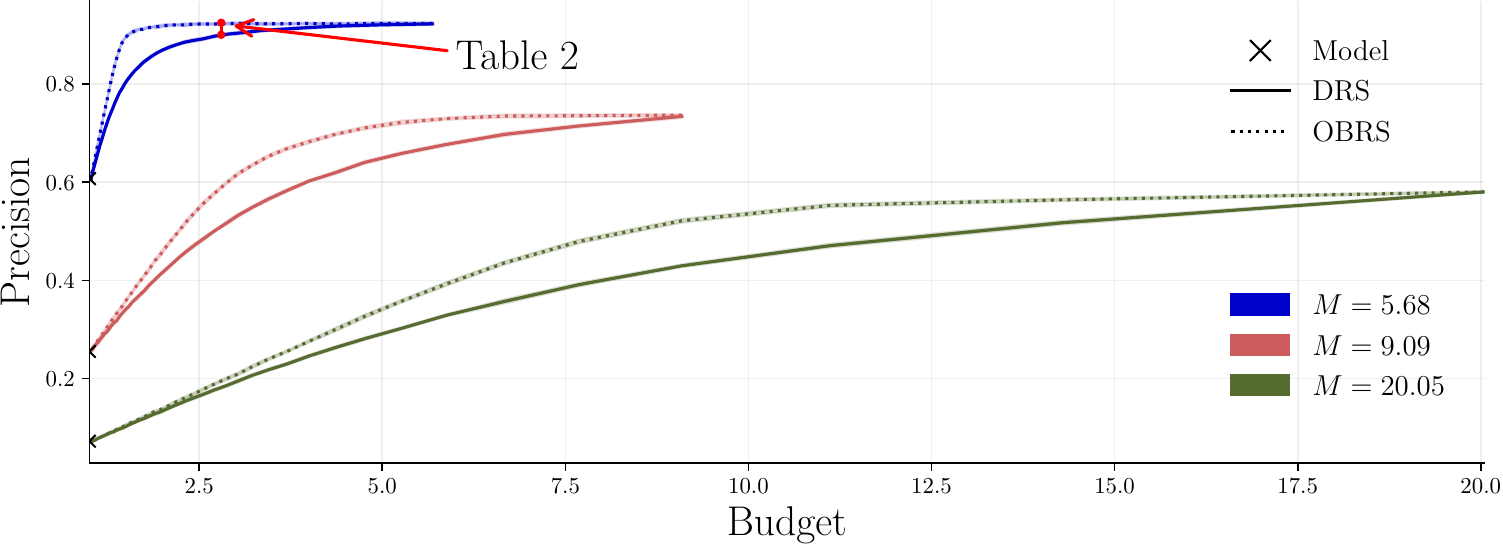}
    \caption{Precision for different budget in various 2D datasets.} 
    \label{app:fig:obrsvsdrs}
\end{figure}

\subsection{Complexity of Algorithm \ref{alg:TOBRS}}\label{app:subsec:complexity}
In Algoritm~\ref{alg:TOBRS}, between every update of $T$ and $G$, the parameter $c_K$ is updated. In practice, the parameter $c_K$ is not update every iterations. In this section, we investigate our the frequency of update affect the training procedure both in convergence speed and in terms of time.

\begin{figure}[H]
    \centering
    \includegraphics[width=\textwidth]{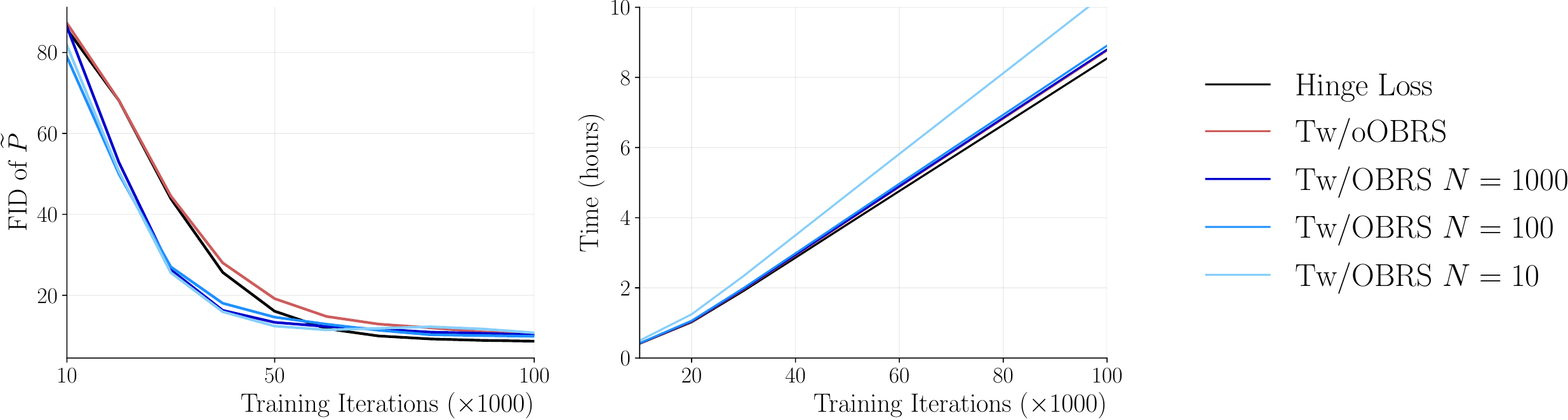}
    \caption{Tw/OBRS: Training BigGAN models on CIFAR-10 with the hinge loss, and $\GAN$ without OBRS (Tw/oOBRS) and with OBRS (Tw/OBRS). For the models trained with OBRS, the parameter $c_K$ is updated every $N$ iterations. }
    \label{app:fig:freqc}
\end{figure}

We train different BigGAN models on CIFAR-10 with different frequency of updates: every $10$ iterations, every $100$ iterations and every $1000$ iterations. We plot in Figure~\ref{app:fig:freqc}, the FID during training and the time during training, both as a function of the number of iterations. We observe that the frequency of updates does not affect the speed of convergence. Furthermore, we observe that update $c_K$ every $10$ operation takes on average $19\%$ longer to train than $\GAN$ without OBRS, while updating every $100$ and $1000$ iterations are only $1.69\%$ and $0.03\%$ longer.

\section{Experimental Procedure}
In this paper, OBRS have been investigated in two different contexts.
\begin{itemize}
    \item Using OBRS to improve a pre-trained model, with different budget.
    \item Training and fine-tuning a model accounting for OBRS with a budget of $K=5$.
\end{itemize}

In every experiment, we have used BigGAN models \cite{brock_large_2019}. For every dataset we have used hyperparameters as close as possible to the original ones. In the original paper, the hinge loss is used and, according to \cite{azadi_discriminator_2019}, the fact that the loss is saturating decreases the performance of the estimation of the density ratio. In the first context, we take a pretrained model, typically trained with hinge loss, and fine-tune the discriminator only, based on $\GAN$. And thus we can perform density estimation for the rejection sampling procedure. In the DRS method, they retrain the discriminator on 10k samples. We opt for training the discriminator on the entire data set with a learning rate of $10^{-10}$ with the same hyperparameters as the one proposed by \cite{brock_large_2019}.

Then for the second context: Tw/OBRS, we need to compare the speed of convergence for three different losses: hinge loss (since it is the original one), $\GAN(P\Vert \whP)$ (Tw/oOBRS) and $\GAN(P\Vert \wtP)$ (Tw/OBRS). However, we evaluate the models based their rejected distribution in terms of FID to analyze the speed of convergence.  Therefore,  two tails for the discriminator were built: one was trained with any loss and the other systematically with $\GAN$. Therefore, we can train the model $G$ based on the given loss and still evaluate the model with OBRS. For the training with OBRS we used this set of hyperparameters. Every model has been trained on a 4xV100 clusters.

\begin{table}[H]
    \centering
    \begin{tabular}{|c|c|ccccc|}
    \hline
       Dataset  & Task  & Tch & Gch & lr T & lr G & Batch Size  \\\hline
        CIFAR-10 & Training & $64$ & $64$ & $2.10^{-5}$ & $2.10^{-5}$ & $50$   \\ 
        CelebA64 & Training & $32$ & $32$ & $1.10^{-4}$ & $4.10^{-4}$ & $128$   \\ 
        CelebA64 & Fine Tuning & $32$ & $32$ & $1.10^{-6}$ & $1.10^{-6}$ & $128$   \\ 
       ImageNet128 & Fine Tuning & $96$ & $96$ & $1.    10^{-5}$ & $1.10^{-5}$ & $2048$   \\ 
\hline
    \end{tabular}
    \caption{Hyper-parameters used for the different BigGAN configurations. Tch and Gch stands for the number of channels in each model. T lr and G lr stands for the learning rate of each models.}
    \label{tab:my_label}
\end{table}

\end{document}